\documentclass[letterpaper]{article} 
\usepackage{aaai24}  
\usepackage{times}  
\usepackage{helvet}  
\usepackage{courier}  
\usepackage[hyphens]{url}  
\usepackage{graphicx} 
\usepackage{natbib}  
\usepackage{caption} 
\newif\ifsubmit  
\submitfalse
\submittrue

\title{  
Return to Tradition: Learning Reliable Heuristics with \\
Classical Machine Learning  
}  

\author{  
Dillon Z. Chen$^{1,2}$, Felipe Trevizan$^{2}$, Sylvie Thi\'ebaux$^{1,2}$\\
}

\affiliations{
    $^{1}$LAAS-CNRS, Universit\'e de Toulouse\quad
    $^{2}$Australian National University\\
    dillon.chen@laas.fr\quad felipe.trevizan@anu.edu.au\quad sylvie.thiebaux@anu.edu.au
}

\usepackage{bm}         
\usepackage[utf8]{inputenc}
\usepackage{lipsum}
\usepackage{amssymb,amsthm,amsmath}
\usepackage{enumitem}
\usepackage{caption}
\usepackage{subcaption}
\usepackage{centernot}
\usepackage{complexity}
\usepackage[ruled,noline,linesnumbered,noend]{algorithm2e} 
\usepackage{mathrsfs}   
\usepackage{tabularx}   
\usepackage{multirow}   
\usepackage{array}      
\usepackage{booktabs}   
\usepackage{colortbl}   
\usepackage{xargs}      
\usepackage{tikz}
\usepackage{tikz-cd}
\usepackage{marginnote} 
\usepackage{afterpage}
\usepackage{pdflscape}

\usetikzlibrary{positioning, calc}

\DeclareMathOperator*{\pre}{pre}
\DeclareMathOperator*{\add}{add}
\DeclareMathOperator*{\del}{del}

\newclass{\N}{N}
\newclass{\CountingLogic}{C}
\newclass{\coNTIME}{coNTIME}
\newclass{\coNSPACE}{coNSPACE}
\newclass{\coNPSPACE}{coNPSPACE}
\newclass{\EXPTIME}{EXPTIME}
\newclass{\NEXPTIME}{NEXPTIME}
\newclass{\coNEXPTIME}{coNEXPTIME}
\newclass{\NEXPSPACE}{NEXPSPACE}
\newclass{\coNEXPSPACE}{coNEXPSPACE}
\newclass{\ASPACE}{ASPACE}
\newclass{\ATIME}{ATIME}
\newclass{\APSPACE}{APSPACE}
\newclass{\AEXPTIME}{AEXPTIME}
\newclass{\AEXPSPACE}{AEXPSPACE}

\newtheorem{theorem}{Theorem}[section]
\newtheorem{corollary}[theorem]{Corollary}

\theoremstyle{definition}
\newtheorem{definition}[theorem]{Definition}

\def\N{\mathbb{N}}
\def\R{\mathbb{R}}

\renewcommand{\phi}{\varphi}

\def\la{\leftarrow}

\newcommand{\abs}[1]{\left| #1 \right|}

\newcommand{\gen}[1]{\left< #1 \right>}

\newcommand{\set}[1]{\left\{ #1 \right\}}
\newcommand{\seta}[1]{\{ #1 \}}
\newcommand{\setbig}[1]{\bigl\{ #1 \bigr\}}

\newcommand{\mset}[1]{\left\{ \!\! \left\{ #1 \right\} \!\! \right\}}
\newcommand{\mseta}[1]{\{ \!\! \{ #1 \} \!\! \}}
\newcommand{\msetaa}[1]{\{ \!\! \{ \! #1 \! \} \!\! \}}

\newcommand{\lr}[1]{\left( #1 \right)}

\newcommand{\Biglr}[1]{\Bigl( #1 \Bigr)}

\newcommand{\header}[1]{\rotatebox[origin=l]{90}{\hspace*{-0.22cm} #1}}
\newcommand{\colorofcell}{blue}
\newcommand{\comparisonentry}[1]{{\tiny{#1}}}
\newcommand{\first}[2]{\cellcolor{\colorofcell!30}{{\textbf{#1}}\comparisonentry{#2}}}
\newcommand{\second}[2]{\cellcolor{\colorofcell!20}{{#1}\comparisonentry{#2}}}
\newcommand{\third}[2]{\cellcolor{\colorofcell!10}{{#1}\comparisonentry{#2}}}
\newcommand{\normalcell}[2]{{{#1}\comparisonentry{#2}}}
\newcommand{\zerocell}[1]{-}

\newcolumntype{Y}{>{\raggedleft\arraybackslash}X}

\newcommand{\tablesize}{\scriptsize}

\newcommand{\significant}[1]{\textbf{#1}}
\newcommand{\notsignificant}[1]{\emph{#1}}

\newcommand{\highcorr}[1]{\cellcolor{green!30}{#1}}
\newcommand{\medcorr}[1]{\cellcolor{green!15}{#1}}
\newcommand{\lowcorr}[1]{\cellcolor{gray!10}{#1}}

\newcommand{\lefttodo}[2][]{{%
 \let\marginpar\marginnote
 \reversemarginpar
 \renewcommand{\baselinestretch}{0.8}%
 \todo[#1]{#2}}
}

\newcommand{\predicates}{\mathcal{P}}
\newcommand{\objects}{\mathcal{O}}
\newcommand{\schemata}{\mathcal{A}}
\newcommand{\neighbour}{\mathcal{N}}
\newcommand{\edge}[1]{\gen{#1}}

\definecolor{caribbeangreen}{rgb}{0.0, 0.8, 0.6}
\definecolor{brilliantlavender}{rgb}{0.96, 0.73, 1.0}
\definecolor{amethyst}{rgb}{0.6, 0.4, 0.8}
\definecolor{ao(english)}{rgb}{0.0, 0.5, 0.0}
\definecolor{arylideyellow}{rgb}{0.91, 0.84, 0.42}
\definecolor{asparagus}{rgb}{0.53, 0.66, 0.42}
\definecolor{aquamarine}{rgb}{0.5, 1.0, 0.83}
\definecolor{babyblue}{rgb}{0.54, 0.81, 0.94}
\definecolor{fwtchanged}{rgb}{0.3, 0.3, 0.7}
\definecolor{rosewood}{rgb}{0.4, 0.0, 0.04}
\definecolor{oldmauve}{rgb}{0.4, 0.19, 0.28}
\definecolor{myrtle}{rgb}{0.13, 0.26, 0.12}
\definecolor{magenta(dye)}{rgb}{0.79, 0.08, 0.48}

\definecolor{plta}{rgb}{0.12, 0.47, 0.71}
\definecolor{pltb}{rgb}{   1, 0.5, 0.05}
\definecolor{pltc}{rgb}{0.17, 0.63, 0.17}
\definecolor{pltd}{rgb}{0.84, 0.15, 0.16}

\newcommand{\hash}{\text{hash}}

\newcommand{\hff}{$h^{\text{FF}}$}

\newcommand{\countfunction}{\texttt{count}}
\newcommand{\Domain}{\mathcal{D}}

\newcommand{\trainingset}{\mathcal{T}_{\Domain}}

\newcommand{\ff}{\hff}
\newcommand{\muninn}{\text{Muninn}}
\newcommand{\gpr}{GPR}
\newcommand{\DescriptionLogic}{DLF}
\newcommand{\lama}{LAMA-first}

\newcommand{\ilg}{\text{ILG}}
\newcommand{\WL}{\mathcal{WLF}^{\ilg{}}}
\newcommand{\GNN}{\mathcal{GN\!N}^{\ilg{}}}
\newcommand{\DL}{\mathcal{DLF}}
\newcommand{\MUNINN}{\muninn}
\newcommand{\params}{\mathbf{\Theta}}
\newcommand{\paramsTwo}{\mathbf{\Phi}}

\newcommand{\svr}{SVR}
\newcommand{\svrOne}{\svr}

\newcommand{\svrInf}{\svr$_{\infty}$} 
\newcommand{\lwlTwo}{\svr$_{\text{2-LWL}}$}

\newcommand{\gprNew}{{\gpr}}

\newcommand{\setofcolours}{\mathcal{C}}  
\newcommand{\featurevec}{\vec{v}}

\newcommand{\mlp}{\textbf{MLP}}

\newcommand{\seeded}{$^\ddagger$}

\newcommand{\gooseMax}{GOOSE$_{\text{max}}$}
\newcommand{\gooseMean}{GOOSE$_{\text{mean}}$}

\renewcommand{\arraystretch}{0.75}
\renewcommand{\arraystretch}{1}

\newcommand{\probliftedgeneral}[2]{\langle\predicates, \objects, \schemata, #1, #2\rangle}
\newcommand{\problifted}{\probliftedgeneral{s_0}{G}}

\renewcommand{\KwData}[1]{}
\renewcommand{\KwResult}[1]{}

\newcommand{\mug}{\text{MuG}} 
\newcommand{\GNNa}{\mathcal{GN\!N}}

\SetInd{0.5em}{0em}

\begin{document}

\maketitle

\begin{abstract}
    Current approaches for learning for planning have yet to achieve competitive performance against classical planners in several domains, and have poor overall performance. In this work, we construct novel graph representations of lifted planning tasks and use the WL algorithm to generate features from them. These features are used with classical machine learning methods which have up to 2 orders of magnitude fewer parameters and train up to 3 orders of magnitude faster than the state-of-the-art deep learning for planning models. Our novel approach, WL-GOOSE, reliably learns heuristics from scratch and outperforms the $h^{\text{FF}}$ heuristic in a fair competition setting. It also outperforms or ties with LAMA on 4 out of 10 domains on coverage and 7 out of 10 domains on plan quality. WL-GOOSE is the first learning for planning model which achieves these feats. Furthermore, we study the connections between our novel WL feature generation method, previous theoretically flavoured learning architectures, and Description Logic Features for planning.
\end{abstract}

\section{Introduction}

Learning for planning has regained traction in recent years due to advancements in deep learning (DL) and neural network architectures.
The focus of learning for planning is to learn domain knowledge in an automated, domain-independent fashion in order to improve the computation and/or quality of plans.
Recent examples of learning for planning methods using DL include learning policies~\cite{toyer2018action,groshev2018learning,garg:etal:2019,rivlin2020generalized,silver:etal:2024}, heuristics~\cite{shen:etal:2020,karia:srivastava:2021} and heuristic proxies~\cite{shen:etal:2019,ferber:etal:2022,chrestien:etal:2023}, with more recent architectures motivated by theory~\cite{stahlberg:etal:2022,stahlberg:etal:2023,mao:etal:2023,chen:etal:2024,horcik:sir:2024}.
However, learning for planning is not a new field and works capable of learning similar domain knowledge using classical statistical machine learning (SML) methods predate DL.
For instance, learning heuristic proxies using support vector machines (SVMs)~\cite{garrett2016learning}, policies using reinforcement learning~\cite{buffet2009factored} and decision lists~\cite{yoon2002inductive}.
We refer to \cite{jimenez2012review} for a more comprehensive overview of classical SML methods.

Unfortunately, all deep learning for planning architectures have yet to achieve competitive performance against classical planners and suffer from a variety of issues including (1) a need to tune a large number of hyperparameters, (2) lack of interpretability and (3) being both data and computationally intensive.
In this paper, we introduce WL-GOOSE, a novel approach for learning for planning that takes advantage of the efficiency of SML-based methods for overcoming all these issues.
WL-GOOSE uses a new graph representation for lifted planning tasks. However, differently from several DL-based methods, we do not use GNNs to learn domain knowledge and use \textit{graph kernels} instead.
More precisely, we use a modified version of the Weisfeiler-Leman algorithm for generating features for graphs~\cite{shervashidze:etal:2011} which can be used to train SML models.
Another benefit of WL-GOOSE is its support for various learning targets, such as heuristic values or policies, without the need for backpropagation to generate features as in DL-based approaches.
We also provide a comprehensive theoretical comparison between our approach, GNNs for learning planning domain knowledge, and Description Logic Features for planning~\cite{martin:geffner:2000}.

To demonstrate the potential of WL-GOOSE, we applied it to learn domain-specific heuristics using two classical SML methods: SVMs and Gaussian Processes (GPs).
We evaluated the learned heuristics against the state-of-the-art learning for planning models on the 2023 International Planning Competition Learning Track benchmarks~\cite{seipp:segovia:2023}.
The learned heuristics generalise better than previous DL-based methods while also being more computationally efficient: our models took less than 15 seconds to train which is up to 3 orders of magnitude times faster than GNNs which train on GPUs.
Furthermore, some of our models were trained in a deterministic fashion with minimal parameter tuning, unlike DL-based approaches which require stochastic gradient descent and tuning of various hyperparameters, and have up to 2 orders of magnitude times fewer learned parameters.
When used with greedy best-first search, our learned heuristic models achieved higher total coverage than \hff{}~\cite{hoffmann:nebel:2001} and vastly outperforms all previous learning for planning models.
Moreover, our learned SVM and GP models outperformed or tied with LAMA~\cite{richter:westphal:2010} on 4 out of 10 domains with regards to coverage, and 7 out of 10 domains for plan quality.
These results make our learned heuristics using WL-GOOSE the first ones to surpass the performance of \hff{} and the best performing learned heuristics against LAMA.

\section{Background and Notation}

\subsubsection{Planning}
A classical planning task~\cite{geffner:bonet:2013} is a state transition model
given by a tuple $\Pi = \gen{S, A, s_0, G}$ where $S$ is a set of states, $A$ a
set of actions, $s_0 \in S$ an initial state and $G \subseteq S$ a set of goal
states. 
An action $a \in A$ is a function $a: S \to S\cup \bot$ where $a(s)=\bot$
indicates that action $a$ is not applicable at state $s$, and otherwise, $a(s)$
is the successor state when $a$ is applied to $s$. 
An action has an associated cost $c(a) \in \R_{\geq 0}$. 
A solution or \emph{plan} for this model is a sequence of actions $\pi = a_1
,\ldots, a_n$ where $s_i = a_i(s_{i-1}) \not=\bot$ for $i=1,\ldots,n$
and $s_n \in G$. 
In other words, a plan is a sequence of applicable actions which progresses the
initial state to a goal state when executed. 
The cost of a plan $\pi$ is the sum of its action costs: $c(\pi) = \sum_{i=1}^n
c(a_i)$. 
A planning task is \emph{solvable} if there exists at least one plan.  
A plan is \emph{optimal} if there does not exist any other plan with 
strictly lower cost.

We represent planning tasks in a compact form which does not require
enumerating all states and actions. 
A \emph{lifted planning task}~\cite{lauer:etal:2021} is a tuple $\Pi =
\problifted$ where $\predicates$ is a set of first-order predicates, 
$\objects$ a set of objects, 
$\schemata$ a set of action schemas, 
$s_0$ the initial state, 
and $G$ the goal condition. 
A predicate $P \in \predicates$ has a set of parameters $x_1, \ldots, x_{n_P}$
where $n_P \in \N$ depends on $P$, and it is possible for a predicate to have
no parameters. 
A ground proposition is a predicate which is instantiated by assigning all of
the $x_i$ with objects from $\objects$ or other defined variables. 
An action schema $a \in \schemata$ is a tuple $\gen{\Delta(a), \pre(a),
\add(a), \del(a)}$ where $\Delta(a)$ is a set of parameter variables, and the
preconditions $\pre(a)$, add effects $\add(a)$, and delete effects $\del(a)$
are sets of predicates from $\predicates$ instantiated with elements from
$\Delta(a) \cup \objects$. 
Each action schema has an associated cost $c(a) \in \R_{\geq 0}$.
An action is an action schema where each variable is instantiated with an
object.
A domain $\Domain$ is a set of lifted planning tasks which share the same sets
of predicates $\predicates$ and action schemas~$\schemata$.

In a lifted planning task, states are represented as sets of ground propositions. 
The following are sets of ground propositions: states, goal condition, and
the preconditions, add effects, and delete effects of all actions. 
An action $a$ is applicable in a state $s$ if $\pre(a) \subseteq s$, in which
case we define $a(s)=(s \setminus \del(a)) \cup \add(a)$. Otherwise
$a(s) = \bot$. The cost of an action is given by the cost of its corresponding
action schema. A state $s$ is a goal state if $G \subseteq s$.

A heuristic is a function $h\!:\!S \!\to\! \R \!\cup\! \set{\infty}$ which maps a state
into a number representing an estimate of the cost of the optimal plan to
the goal, or $\infty$ representing the state is unsolvable. 
A heuristic can be defined on problems by evaluating their initial state: $h(\Pi) = h(s_0)$. 
The optimal heuristic $h^*$ returns for each state $s$ the cost of the optimal
plan to the goal if the problem is solvable from $s$, and $\infty$ otherwise.

\subsubsection{The Weisfeiler-Leman algorithms}
We write $\gen{V, E, c, l}$ for a graph with coloured nodes and edges, where
$V$ is a set of nodes, $E \subseteq {{V}\choose{2}}$ is a set of undirected
edges, $c:V \to \Sigma_V$ maps nodes to a set of colours $\Sigma_V$, and $l: E
\to \Sigma_E$ maps edges to a set of colours $\Sigma_E$. 
The edge neighbourhood of a node $u$ under edge colour $\iota$ is
$\neighbour_{\iota}(u) = \set{e=\edge{u,v}=\edge{v,u} \in E \mid l(e) =
\iota}$. 
The neighbourhood of a node $u$ in a graph is $\neighbour(u) = \bigcup_{\iota
\in \Sigma_E} \neighbour_{\iota}(u)$.

\begin{algorithm}[t]
    \caption{WL algorithm}\label{alg:wl}  
    \KwData{A graph $G=\gen{V, E, c}$ and number of iterations $h$.}  
    \KwResult{Multiset of colours.}  
    $c^{0}(v) \la c(v), \forall v \in V$ \label{line:wl:init}\\
    \For{$j=1,\ldots,L$ \normalfont{\textbf{do for}} $v \in V$}{ 
      $c^{j}(v) \la 
      \hash
      \lr{c^{j-1}(v), \mseta{c^{j-1}(u) \mid u \in \neighbour(v)}}  
      $ \label{line:wl:update}
    } 
    \Return{$\bigcup_{j=0,\ldots,L}\mseta{c^{j}(v) \mid v \in V}$} \label{line:wl:return}
\end{algorithm}

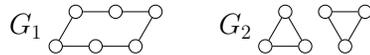
\begin{figure} \centering \resizebox{0.28\textwidth}{!}{%
\newcommand{\wlfontsize}{\Huge}
\begin{tikzpicture} 
\node (aaa) at (-1.5, 1) {\wlfontsize $G_1$};
\newcommand{\shift}{-4.25}
\node[draw, circle] (A) at (3.5+\shift,0.5) {} ; 
\node[draw, circle] (B) at (4.5+\shift,0.5) {} ; 
\node[draw, circle] (C) at (4+\shift,1.366) {} ; 
\node[draw, circle] (D) at (5.5+\shift,0.5) {} ; 
\node[draw, circle] (E) at (5+\shift,1.366) {} ; 
\node[draw, circle] (F) at (6+\shift,1.366) {} ;

\draw (A) -- (C) -- (E) -- (F) -- (D) -- (B) -- (A) -- cycle;

\renewcommand{\shift}{1}
\node (aaa) at (2.75+\shift, 1) {\wlfontsize $G_2$}; \node[draw, circle] (G) at (3.5+\shift,0.5) {} ; 
\node[draw, circle] (H) at (4.5+\shift,0.5) {} ; 
\node[draw, circle] (I) at (4+\shift,1.366) {} ; 
\node[draw, circle] (J) at (5.5+\shift,0.5) {} ; 
\node[draw, circle] (K) at (5+\shift,1.366) {} ; 
\node[draw, circle] (L) at (6+\shift,1.366) {} ;

\draw (G) -- (H) -- (I) -- (G) -- cycle; 
\draw (J) -- (K) -- (L) -- (J) -- cycle;
\end{tikzpicture}
} \caption{Two non-isomorphic graphs $G_1$ (6-cycle) and $G_2$ (two disjoint
3-cycles) which the WL algorithm returns the same outputs, thus failing to
distinguishing them.}\label{fig:example}
\end{figure}

We only focus on the WL algorithm which is a special case of the class of
$k$-Weisfeiler-Leman ($k$-WL) algorithms~\cite{lehman:weisfeiler:1968}.
The $k$-WL algorithms were originally constructed to provide tests for whether
pairs of graphs are isomorphic or not.
The $k\!+\!1$-WL algorithm subsumes the $k$-WL algorithm as it can distinguish
a greater class of non-isomorphic graphs, and furthermore is in correspondence
with $k$-variable counting logics~\cite{cai:etal:1992}.
However, the complexity of the $k$-WL algorithms is exponential in $k$.

The WL algorithm takes as input graphs without edge colours, i.e.  $\forall
e\in E, l(e) = 0$, and outputs a canonical form in terms of a multiset of
colours, a set which is allowed to have duplicate elements.
It has also been used to construct a kernel for
graphs~\cite{shervashidze:etal:2011} which converts the multiset of colours in
the WL algorithm into a feature vector and then uses the simple dot product
kernel.
We denote a multiset of elements by $\mset{\ldots}$.

The WL algorithm is shown in Alg.~\ref{alg:wl} which takes as input a graph $G$
with coloured nodes only and a predefined number of WL iterations $L$.
The algorithm begins by initialising the current colours of each node with the
initial node colours.
If no node colours are given in the graph, we can set them to 0.
Line~\ref{line:wl:update} updates the colour of each node $v$ by iteratively
collecting the current colours of its neighbors in a multiset and then hashing
this multiset and $v$'s current colour into a colour using an injective
$\hash(\cdot,\cdot)$ function.
In practice, $\hash$ is built lazily by using a map data structure and
multisets are represented as sorted strings.
Line~\ref{line:wl:return} returns a multiset of the node colours seen over all iterations.

If the WL algorithm outputs two different multisets for two graphs $G_1$ and
$G_2$, then the graphs are non-isomorphic.
However, if the algorithm outputs the same multisets for two graphs we cannot
say for sure whether they are isomorphic or not.
The canonical example illustrating this case is in Fig.~\ref{fig:example} where
the two graphs are not isomorphic but the WL algorithm returns the same output
for both graphs since it views all nodes as the same because they have degree
2.

\section{WL Features for Planning} \label{sec:method}

In this section we describe how to generate features for planning states in order to learn heuristics.
The process involves three main steps:
(1) converting planning states into graphs with coloured nodes and edges,
(2) running a variant of the WL algorithm on the graphs in order to generate features, and then
(3) training a classical machine learning model for predicting heuristics using the obtained features.
We start by defining the Instance Learning Graph (\ilg{}), a novel representation for lifted planning tasks.

\newcommand{\xxshift}{1.5cm}
\newcommand{\yyshift}{0.8cm}
\newcommand{\csize}{0.8cm}
\newcommand{\descccsize}{\footnotesize}
\newcommand{\ilgword}[1]{\text{\emph{\descccsize{#1}}}}
\newcommand{\desccccc}{5cm}
\newcommand{\desccc}{6cm}
\newcommand{\anchor}{west}
\begin{figure}
\centering
\begin{tikzpicture}[thick,
every node/.style={  
    draw,  
    rounded corners,  
}  
]
\node[fill=caribbeangreen] (onab) at (-2.5/1.8*\xxshift,\yyshift) {
$\ilgword{on(a,b)}$
};
\node[fill=yellow] (onca) at ( 2.5/1.8*\xxshift,\yyshift) {
$\ilgword{on(c,a)}$
};
\node[fill=caribbeangreen] (onbc) at ( 0,\yyshift) {
$\ilgword{on(b,c)}$
};
\node[fill=babyblue] (a) at (  -2.5/1.8*\xxshift,2*\yyshift) {
$\ilgword{a}$
};
\node[fill=babyblue] (b) at (   0/1.8*\xxshift,2*\yyshift) {
$\ilgword{b}$
};
\node[fill=babyblue] (c) at (   2.5/1.8*\xxshift,2*\yyshift) {
$\ilgword{c}$
};


\path [-,draw=pltb] (b.south) edge (onab.north);
\path [-,draw=pltb] (c.south) edge (onbc.north);
\path [-,draw=pltb] (a.south) edge (onca.north);
\path [-,draw=plta] (c.south) edge (onca.north);
\path [-,draw=plta] (a.south) edge (onab.north);
\path [-,draw=plta] (b.south) edge (onbc.north);
\end{tikzpicture}
\caption{\ilg{} subgraph of facts and goal condition corresponding to the $\emph{on}$ predicate of a Blocksworld instance. The current state says that $a$ stacked on $b$, which is stacked on $c$, and the goal condition is for $c$ to be stacked on $a$.}\label{fig:llg}
\end{figure}
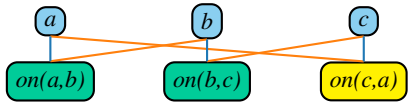
\newcommand{\objectsymbol}{\texttt{ob}} \newcommand{\truesymbol}{\texttt{ap}}
\newcommand{\goalsymbol}{\texttt{ug}} \newcommand{\truegoalsymbol}{\texttt{ag}}
\begin{definition}\label{def:llg} The \emph{instance learning graph (\ilg{})}
of a lifted planning problem $\Pi=\problifted$ is the graph $G=\gen{V, E, c,
l}$ with

$\bullet$ $V = \objects \cup s_0 \cup G$

$\bullet$ $E =
\bigcup_{ p=P(o_1,\ldots,o_{n_P}) \in s_0 \cup G }
\setbig{ \edge{p, o_1},\ldots,\edge{p, o_{n_P}} }$

$\bullet$ $c: V \to (\set{\truesymbol, \goalsymbol, \truegoalsymbol} \times
\predicates) \cup \set{\objectsymbol}$ defined by
\begin{align*}
u \mapsto
\begin{cases}
\objectsymbol, &\text{if $u \in \objects$;} \\
(\truegoalsymbol, P), &\text{if $u=P(o_1,\ldots,o_{n_P}) \in s_0 \cap G$;} \\
(\truesymbol, P), &\text{if $u=P(o_1,\ldots,o_{n_P}) \in s_0 \setminus G$;} \\
(\goalsymbol, P), &\text{if $u=P(o_1,\ldots,o_{n_P}) \in G \setminus s_0$;}
\end{cases}
\end{align*}

$\bullet$ $l: E \to \N$ with $\edge{p, o_i} \mapsto i.$
\end{definition}

Fig.~\ref{fig:llg} provides an example of an \ilg{}.
An ILG consists of a node for each object and the union of propositions that are true in the state $s_0$ and the goal condition $G$.
A proposition is connected to the $n$ object nodes which instantiates the proposition.
The labels of the $n$ edges correspond to the position of the object in the predicate argument.
The colours of the nodes indicate whether the node corresponds to an object ($\objectsymbol$), or determines whether it is a proposition belonging to $s_0$ ($\truesymbol$) or $G$ ($\goalsymbol$) only or both ($\truegoalsymbol$), as well as its corresponding predicate. 
Hence $\goalsymbol$ stands for unachieved goal, $\truegoalsymbol$ for achieved goal, and $\truesymbol$ for achieved (non-goal) proposition.
Note that ILGs are agnostic to the transition system of the planning task as they ignore action schemas and actions.

Since ILGs have coloured edges, we need to extend the WL algorithm to account for edge colours to generate features for ILGs.
Our modified WL algorithm is obtained by replacing Line 3
in Alg.~\ref{alg:wl} with the update function
\begin{align*}
c^{j}(v)  \la  \hash
\Biglr{
c^{j-1}(v),  \bigcup_{\iota \in \Sigma_E}  \!\!\!
\mset{
    (c^{j-1}(u), \iota) \mid  u \in \neighbour_{\iota}(v)
}
},
\end{align*}

\noindent
where the union of multisets is itself a multiset.
Note that edge colours do not update during this modified WL algorithm.
It is possible to run a variant of the WL algorithm which modifies edge colours but this comes at an additional computational cost given that usually $\abs{E} \gg \abs{V}$.

Now that ILGs can be represented as multisets of colours, we can generate features by representing these multisets as histograms~\cite{shervashidze:etal:2011}.
The feature vector of a graph is a vector $v$ with a size equal to the number of observed colours during training, where $v[\kappa]$ counts how many times the WL algorithm has encountered colour $\kappa$ throughout its iterations.
Formally, let $G_1=\gen{V_1, E_1, c_1, l_1},\ldots,G_n=\gen{V_n, E_n, c_n, l_n}$ be the set of training graphs.
Then the colours the WL algorithm encounters in the training graphs are given by
\begin{align*}
\setofcolours = \seta{c_i^{j}(v) \mid i \in \{1,\ldots,n\}; j \in \{0,\ldots,h\}; v \in V_i} 
\end{align*}
where $c_i^{j}(v)$ is the colour of node $v$ in graph $G_i$ during the $j$-th iteration of WL for $j>0$ and $c_i^{0}(v) = c_i(v)$.
Given a planning task $\Pi$ and the set of colours $\setofcolours$ observed during training,
$\Pi$'s feature vector representation $\vec{v} \in \R^{|\setofcolours|}$ is $\featurevec=[\countfunction_{\setofcolours}(\Pi, \kappa_1),\dots,\countfunction_{\setofcolours}(\Pi, \kappa_{|\setofcolours|})]$ where $\countfunction_{\setofcolours}(\Pi, \kappa)$ is the number of times the colour $\kappa \in \setofcolours$ is present in the output of the WL algorithm on the ILG representation of $\Pi$.
There is no guarantee that $\setofcolours$ contains all possible observable colours for a given planning domain and colours not in $\setofcolours$ observed after training are ignored.

\section{Theoretical Results} \label{sec:connections}

In this section, we investigate the relationship between WL features and features generated using message passing graph neural network (GNN), description logic features (\DescriptionLogic{}) for planning~\cite{martin:geffner:2000} and the features used by Muninn~\cite{stahlberg:etal:2022,stahlberg:etal:2023}, a theoretically motivated deep learning model.
Fig.~\ref{fig:expressivity} summarises our theoretical results.


\begin{figure}
\centering  
\newcommand{\tikzlengthx}{3}  
\newcommand{\tikzlengthy}{1.1}  
\resizebox{0.95\columnwidth}{!}{%
\begin{tikzpicture}[  
    every node/.style={  
        align=center,  
        draw,  
        rounded corners,  
        minimum height=0.8cm,  
        minimum width=1.6cm,   
    }  
]  

\node at (+\tikzlengthx, \tikzlengthy) (DL) {$\DL$};  
\node at (0*\tikzlengthx, \tikzlengthy) (WL) {$\WL_{\phantom{\ilg{}}}$};  
\node at (-\tikzlengthx, \tikzlengthy) (GNN) {$\GNN_{\phantom{\ilg{}}}$};  
\node at (-\tikzlengthx, 0) (Muninn) {$\MUNINN$};  

\draw[dashed]   
(WL) --  
node[midway, xshift=0em, yshift=1em, draw=none] {$\not=$}   
node[midway, xshift=0em, yshift=2em, draw=none] {Thm.~\ref{thm:wl_dl}}   
(DL);  

\draw[<->]   
(WL) --  
node[midway, xshift=0em, yshift=2em, draw=none] {Thm.~\ref{thm:wl_gnn}}   
node[midway, xshift=0em, yshift=0.85em, draw=none] {$=$}   
(GNN);

\draw[<-]   
(GNN) to[out=180, in=180]
node[midway, xshift=-2.5em, yshift=-0.5em, draw=none] {$\subsetneq$}   
node[midway, xshift=-2.5em, yshift= 0.5em, draw=none] {Thm.~\ref{thm:gnn_muninn}}   
(Muninn);

\draw[->]   
(Muninn.east) to[out=0, in=270] 
node[midway, xshift=4.5em, yshift=-0.5em, draw=none] {$\subsetneq$}   
node[midway, xshift=4.5em, yshift= 0.5em, draw=none] {Cor.~\ref{thm:wl_muninn}}   
(WL.south);


\end{tikzpicture}
}
\caption{  
    Expressivity hierarchy of WL, GNN and DL generated features for planning. 
}  
\label{fig:expressivity}
\end{figure}
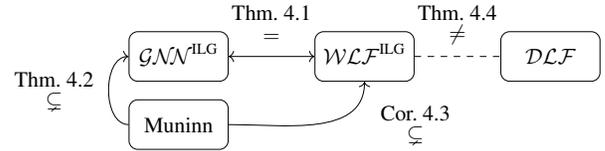

We begin with some notation.
Let $\Domain$ represent the set of all problems in a given domain.
We define $\WL_{\params}\!:\!\Domain \to \R^d$ as the WL feature generation
function described in Sec.~\ref{sec:method} which runs the WL algorithm on the
ILG representation of planning tasks.
We denote $\params$ the set of parameters of the function which includes the
number of WL iterations and the set of colours $\setofcolours$ with size $d$ observed during training.
We similarly denote parametrised GNNs acting on ILG representations of
planning tasks by $\GNN_{\params}\!:\!\Domain \to \R^d$.
Parameters for GNNs include number of message passing layers, the message
passing update function with fixed weights, and the aggregation function.

We denote \DescriptionLogic{} generators~\cite{martin:geffner:2000} by $\DL_{\params}\!:\!\Domain \to
\R^d$ where the parameters for $\DL$ include the maximum complexity length of
its features.
\DescriptionLogic{}s have been used in several areas of learning for planning including learning descending dead-end avoiding heuristics~\cite{frances:etal:2019}, unsolvability heuristics~\cite{staahlberg:etal:2021} and policy sketches~\cite{bonet:etal:2019,drexler:etal:2022}.
Lastly, we denote the architecture from~\citet{stahlberg:etal:2022} for
generating features by $\MUNINN_{\params}\!:\!\Domain \to \R^d$.
We omit their final MLP layer which transforms the vector feature into a
heuristic estimate.
Furthermore in our theorems, we ignore their use of random node initialisation
(RNI)~\cite{abboud:etal:2021}.
The original intent of RNI is to provide a universal approximation theorem for
GNNs but the practical use of the theorem is limited by the assumption of
exponential width layers and absence of generalisation results.
Parameters for $\MUNINN$ include hyperparameters for their GNN architecture and
learned weights for their update functions.

In all of the aforementioned models, the parameters
$\params$ consist of a combination of model hyperparameters and trained
parameters based on a training set $\trainingset \subseteq \Domain$.
The expressivity and distinguishing power of a feature generator for planning determines if it can theoretically learn $h^*$ for larger subsets of planning tasks.
We begin with an application of a well-known result connecting the expressivity
of the WL algorithm and GNNs for distinguishing graphs~\cite{xu:etal:2019} by extending it to edge-labelled graphs.

\begin{theorem}
    [$\WL$ and $\GNN$ have the same power at distinguishing planning tasks.]  
    \label{thm:wl_gnn}
    Let $\Pi_1$ and $\Pi_2$ be any two planning tasks from a given domain.
    If for a set of parameters $\params$ we have that
    $\GNN_{\params}(\Pi_1)\not=\GNN_{\params}(\Pi_2)$, then there exists a
    corresponding set of parameters $\paramsTwo$ such that
    $\WL_{\paramsTwo}(\Pi_1)\not=\WL_{\paramsTwo}(\Pi_2)$.
    Conversely for all $\paramsTwo$ such that $\WL_{\paramsTwo}(\Pi_1) \not=
    \WL_{\paramsTwo}(\Pi_2)$, there exists $\params$ such that
    $\GNN_{\params}(\Pi_1)\not=\GNN_{\params}(\Pi_2)$.
\end{theorem}
\begin{proof}[Proof]
    [$\subseteq$] 
    The forward statement follows from~\cite[Lemma 3]{xu:etal:2019} which
    states that GNNs are at most as expressive as the WL algorithm for
    distinguishing non-isomorphic graphs.
    We can modify the lemma for the edge labelled WL algorithm
    and GNNs which account for edge features.
    Then the result follows after performing the transformation of planning
    tasks into the ILG representation.
    
    [$\supseteq$]
    The converse statement follows from~\cite[Corollary 6]{xu:etal:2019} and
    modifying Eq.~(4.1) of their GIN architecture by introducing an MLP for
    each of the finite number of edge labels in the ILG graph and summing their
    outputs at each GIN layer.
    The MLPs have disjoint range in order for injectivity to be preserved
    as to achieve the same distinguishing power of the edge labelled WL
    algorithm.
    This can be easily enforced by increasing the hidden dimension size and
    having each MLP to map to orthogonal dimensions.
\end{proof}

We proceed to show that GNNs acting on ILGs is similar to Muninn's GNN architecture~\cite{stahlberg:etal:2022}.
The idea of the proof is that encoding different predicates into the ILG
representation is equivalent to having different weights for message passing to
and from different predicates in Muninn.
However, we also show that our model has strictly higher expressivity for
distinguish planning tasks due to explicitly encoding achieved goals.

\begin{theorem}
    [$\GNN$ is strictly more expressive than $\MUNINN$ at distinguishing
    planning tasks.]  
    \label{thm:gnn_muninn}
    Let $\Pi_1$ and $\Pi_2$ be any two planning tasks from a given domain.
    For all $\params$, if
    $\MUNINN_{\params}(\Pi_1)\not=\MUNINN_{\params}(\Pi_2)$, then there exists
    a corresponding set of parameters $\paramsTwo$ such that
    $\GNN_{\paramsTwo}(\Pi_1)\not=\GNN_{\paramsTwo}(\Pi_2)$.
    Furthermore, there exists a pair of planning tasks $\Pi_1$ and $\Pi_2$ such
    that there exists $\paramsTwo$ with
    $\GNN_{\paramsTwo}(\Pi_1)\not=\GNN_{\paramsTwo}(\Pi_2)$ but for all
    $\params$, $\MUNINN_{\params}(\Pi_1)=\MUNINN_{\params}(\Pi_2)$.
\end{theorem}

\begin{proof}[Proof]
    \renewcommand{\xxshift}{0.7cm}
\renewcommand{\yyshift}{0.9cm}
\renewcommand{\csize}{0.8cm}
\renewcommand{\descccsize}{\scriptsize}
\newcommand{\objectsize}{\descccsize}
\renewcommand{\ilgword}[1]{\text{\descccsize{#1}}}
\newcommand{\objword}[1]{\text{\objectsize{#1}}}
\renewcommand{\desccccc}{5cm}
\renewcommand{\desccc}{6cm}
\renewcommand{\anchor}{west}
\newcommand{\minimumheight}{1.5em}

\begin{figure}

\newcommand{\subcapa}{\ilg{} of $\Pi_1$}
\newcommand{\subcapb}{\ilg{} of $\Pi_2$}
\newcommand{\subcapc}{Implicit Muninn graph of $\Pi_1$}
\newcommand{\subcapd}{Implicit Muninn graph of $\Pi_2$}
\newcommand{\wholecap}{\ilg{} and Muninn graph representations of tasks in Thm.~\ref{thm:gnn_muninn} [$\supsetneq$].}
\centering
\begin{subfigure}{0.49\columnwidth}
    \centering
    \begin{tikzpicture}[  
        thick,  
        every node/.style={  
            draw,  
            rounded corners,  
            minimum height=\minimumheight,  
        }  
    ]
    \node[fill=caribbeangreen] (onaa) at (-2.25*\xxshift,\yyshift) {
    $\ilgword{Q(a,a)}$
    };
    \node[fill=caribbeangreen] (onbb) at (-0.75*\xxshift,\yyshift) {
    $\ilgword{Q(b,b)}$
    };
    \node[fill=yellow] (onab) at ( 0.75*\xxshift,\yyshift) {
    $\ilgword{Q(a,b)}$
    };
    \node[fill=yellow] (onba) at ( 2.25*\xxshift,\yyshift) {
    $\ilgword{Q(b,a)}$
    };
    \node[fill=babyblue] (a) at (  -1.5*\xxshift,2*\yyshift) {
    $\objword{a}$
    };
    \node[fill=babyblue] (b) at (   1.5 *\xxshift,2*\yyshift) {
    $\objword{b}$
    };
    
    \path [-,draw=plta] (a.south) edge[bend left=15] (onaa.north);
    \path [-,draw=pltb] (a.south) edge[bend right=15] (onaa.north);
    \path [-,draw=plta] (b.south) edge[bend left=15] (onbb.north);
    \path [-,draw=pltb] (b.south) edge[bend right=15] (onbb.north);
    
    \path [-,draw=plta] (a.south) edge (onab.north);
    \path [-,draw=pltb] (b.south) edge (onab.north);
    
    \path [-,draw=plta] (b.south) edge (onba.north);
    \path [-,draw=pltb] (a.south) edge (onba.north);

    \end{tikzpicture}
    \caption{\subcapa}
\end{subfigure}
\begin{subfigure}{0.49\columnwidth}
    \centering
    \begin{tikzpicture}[  
        thick,  
        every node/.style={  
            draw,  
            rounded corners,  
            minimum height=\minimumheight,  
        }  
    ]
    \node[fill=aquamarine] (onab) at (-1.5*\xxshift,\yyshift) {
    $\ilgword{Q(a,b)}$
    };
    \node[fill=aquamarine] (onba) at ( 1.5*\xxshift,\yyshift) {
    $\ilgword{Q(b,a)}$
    };
    \node[fill=babyblue] (a) at (  -1.5*\xxshift,2*\yyshift) {
    $\objword{a}$
    };
    \node[fill=babyblue] (b) at (   1.5 *\xxshift,2*\yyshift) {
    $\objword{b}$
    };
    
    \path [-,draw=plta] (a.south) edge (onab.north);
    \path [-,draw=pltb] (b.south) edge (onab.north);
    
    \path [-,draw=plta] (b.south) edge (onba.north);
    \path [-,draw=pltb] (a.south) edge (onba.north);

    \end{tikzpicture}
    \caption{\subcapb}
\end{subfigure}
\begin{subfigure}{0.49\columnwidth}
    \centering
    \begin{tikzpicture}[  
        thick,  
        every node/.style={  
            draw,  
            rounded corners,  
            minimum height=\minimumheight,  
        }  
    ]
    \node[fill=gray!30] (onaa) at (-2.25*\xxshift,\yyshift) {
    $\ilgword{Q(a,a)}$
    };
    \node[fill=gray!30] (onbb) at (-0.75*\xxshift,\yyshift) {
    $\ilgword{Q(b,b)}$
    };
    \node[fill=gray!30] (onab) at ( 0.75*\xxshift,\yyshift) {
    $\ilgword{Q}_{\ilgword{g}}\ilgword{(a,b)}$
    };
    \node[fill=gray!30] (onba) at ( 2.25*\xxshift,\yyshift) {
    $\ilgword{Q}_{\ilgword{g}}\ilgword{(b,a)}$
    };
    \node[fill=gray!30] (a) at (  -1.5*\xxshift,2*\yyshift) {
    $\objword{a}$
    };
    \node[fill=gray!30] (b) at (   1.5 *\xxshift,2*\yyshift) {
    $\objword{b}$
    };
    
    \path [-,draw=plta] (a.south) edge[bend left=15] (onaa.north);
    \path [-,draw=pltb] (a.south) edge[bend right=15] (onaa.north);
    \path [-,draw=plta] (b.south) edge[bend left=15] (onbb.north);
    \path [-,draw=pltb] (b.south) edge[bend right=15] (onbb.north);
    
    \path [-,draw=pltc] (a.south) edge (onab.north);
    \path [-,draw=pltd] (b.south) edge (onab.north);
    
    \path [-,draw=pltc] (b.south) edge (onba.north);
    \path [-,draw=pltd] (a.south) edge (onba.north);

    \end{tikzpicture}
    \caption{\subcapc}
\end{subfigure}
\begin{subfigure}{0.49\columnwidth}
    \centering
    \begin{tikzpicture}[  
        thick,  
        every node/.style={  
            draw,  
            rounded corners,  
            minimum height=\minimumheight,  
        }  
    ]
    \node[fill=gray!30] (onab) at (-2.25*\xxshift,\yyshift) {
    $\ilgword{Q(a,b)}$
    };
    \node[fill=gray!30] (onba) at (-0.75*\xxshift,\yyshift) {
    $\ilgword{Q(b,a)}$
    };
    \node[fill=gray!30] (onabg) at ( 0.75*\xxshift,\yyshift) {
    $\ilgword{Q}_{\ilgword{g}}\ilgword{(a,b)}$
    };
    \node[fill=gray!30] (onbag) at ( 2.25*\xxshift,\yyshift) {
    $\ilgword{Q}_{\ilgword{g}}\ilgword{(b,a)}$
    };
    \node[fill=gray!30] (a) at (  -1.5*\xxshift,2*\yyshift) {
    $\objword{a}$
    };
    \node[fill=gray!30] (b) at (   1.5 *\xxshift,2*\yyshift) {
    $\objword{b}$
    };
    
    \path [-,draw=plta] (a.south) edge (onab.north);
    \path [-,draw=pltb] (b.south) edge (onab.north);
    \path [-,draw=plta] (b.south) edge (onba.north);
    \path [-,draw=pltb] (a.south) edge (onba.north);
    
    \path [-,draw=pltc] (a.south) edge (onabg.north);
    \path [-,draw=pltd] (b.south) edge (onabg.north);
    
    \path [-,draw=pltc] (b.south) edge (onbag.north);
    \path [-,draw=pltd] (a.south) edge (onbag.north);

    \end{tikzpicture}
    \caption{\subcapd}
\end{subfigure}
\caption{\wholecap}  
\label{fig:gnn_muninn}
\end{figure}

    [$\supseteq$]  
    In order to show the inclusion, we show that a $\MUNINN$ instance operating
    on a state can be expressed as a GNN operating on the \ilg{} representation
    of the state.
    %
    More explicitly, we show that the implicit graph representation of planning
    states by $\MUNINN$ is the same graph as \ilg{}.
    The message passing steps and initial node features are different but the
    semantic meaning of executing both algorithms are the same.
    The node features in the implicit graphs of Muninn are all the same when
    ignoring random node initialisation.
    Muninn differentiates object nodes and fact nodes by using different
    message passing functions depending on whether a node is an object or a
    fact, and depending on which predicate the fact belongs to.
    %
    In the language of \ilg{}, $\MUNINN$'s message passing step on fact
    nodes $p=P(o_1,\ldots,o_{n_P})$ is 
    \begin{align}
        h_p^{L+1} = \mlp_P(h_{o_1}^{L}, \ldots, h_{o_{n_P}}^{L}) \label{eq:muninn_fact}
    \end{align}
    where $h_p^{L+1}$ denotes the latent embedding of the node $p$ in the
    $L+1$-th layer, $h_{o_i}^{L}$ denotes the latent embedding of the
    object node $o_i$ in the $l$-th layer, and $\mlp_P$ is a multilayer
    perceptron, with a different one for each predicate.
    The message passing step of $\MUNINN$ on object nodes $o_i$ is
    \begin{align}
        h_o^{L+1} = \mlp_U(h_o^{L}, \mseta{h_{p}^{L} \mid o \in p}) \label{eq:muninn_object}
    \end{align}
    where $o \in p$ denotes that $o$ is an argument of the predicate associated
    with $p$.
    We note that having a different $\mlp$ in the message passing step for
    different nodes is equivalent to having a larger but
    identical $\mlp$ in the
    message passing step for all nodes.
    This is because the model can learn to partition latent node features
    depending on their semantic meaning and thus be able to use a single $\mlp$
    function to act as multiple different functions for different node feature
    partitions.
    Thus, Eq.~\eqref{eq:muninn_fact} and~\eqref{eq:muninn_object} can be
    imitated by a GNN operating on \ilg{} since \ilg{} features differentiate
    nodes depending on whether
    they correspond to an object, or a fact
    associated with a predicate.
    Different edge labels in the \ilg{} allow it to distinguish
    the relationship between facts and objects depending on
    their position in the
    predicate argument.

    
    [$\supsetneq$] 
    The main idea here is that Muninn does not keep track of achieved goals and
    sometimes cannot even see that the goal has been achieved.
    Firstly, let \mug{} denote the underlying edge-labelled graph representation of planning tasks in Muninn, such that $\GNNa^{\mug} = \MUNINN$.
    To see how $\GNN$ is strictly more expressive than Muninn,
    we consider the following pair of planning tasks.
    Let $\Pi_1 = \probliftedgeneral{s_0^{(1)}}{G}$ and $\Pi_2 =
    \probliftedgeneral{s_0^{(2)}}{G}$ with $\predicates = \set{Q}$,
    $\objects=\set{a, b}$, $\schemata = \emptyset$, $G = s_0^{(2)} = \set{Q(a,
    b), Q(b, a)}$ and $s_0^{(1)} = \set{Q(a,a), Q(b,b)}$.
    Fig.~\ref{fig:gnn_muninn} illustrates the \ilg{} and \mug{} representation of $\Pi_1$ and $\Pi_2$.
    It is clear that the \ilg{} representation of $\Pi_1$ and $\Pi_2$ are
    different and hence $\GNN$ differentiates between $\Pi_1$ and $\Pi_2$.
    On the other hand without RNI, any edge-labelled variant of the WL algorithm views the \mug{} representation of $\Pi_1$ and $\Pi_2$ illustrated in Fig.~\ref{fig:gnn_muninn}(c) and (d) as the same.
    Thus, $\GNNa^{\mug} = \MUNINN$ views the graphs as the same.
\end{proof}

\begin{corollary}
    [$\WL$ is strictly more expressive than $\MUNINN$ at distinguishing planning
    tasks.]  
    \label{thm:wl_muninn}
    Let $\Pi_1$ and $\Pi_2$ be any two planning tasks from a given domain.
    For all $\params$, if $\MUNINN_{\params}(\Pi_1)\not=\MUNINN_{\params}(\Pi_2)$,
    then there exists a corresponding set of parameters $\paramsTwo$ such that
    $\WL_{\paramsTwo}(\Pi_1)\not=\WL_{\paramsTwo}(\Pi_2)$.
    Furthermore, there exists a pair of planning tasks $\Pi_1$ and $\Pi_2$ such
    that there exists $\paramsTwo$ with
    $\WL_{\paramsTwo}(\Pi_1)\not=\WL_{\paramsTwo}(\Pi_2)$ but for all $\params$,
    $\MUNINN_{\params}(\Pi_1)=\MUNINN_{\params}(\Pi_2)$.
\end{corollary}


Our next theorem shows that $\WL$ and $\DL$ features are incomparable, in the
sense that there are pairs of planning tasks that look equivalent to one model
but not the other.
We use a similar counterexample to that used for Muninn but with an extra
predicate which $\WL$ does not distinguish but $\DL$ can.
Conversely we use the fact that \DescriptionLogic{}s are limited by the need to
convert planning predicates into binary predicates to construct a
counterexample pair of planning tasks with ternary predicates which $\DL$ views
as the same while $\WL$ does not.

\begin{theorem}
    [$\WL$ and $\DL$ are incomparable at distinguishing planning tasks.]  
    \label{thm:wl_dl}
    There exists a pair of planning tasks $\Pi_1$ and $\Pi_2$ such that there
    exists $\paramsTwo$ with
    $\WL_{\paramsTwo}(\Pi_1)\not=\WL_{\paramsTwo}(\Pi_2)$ but for all
    $\params$, $\DL_{\params}(\Pi_1)=\DL_{\params}(\Pi_2)$.
    Furthermore, there exists a pair of planning tasks $\Pi_1$ and $\Pi_2$ such
    that there exists $\paramsTwo$ with
    $\DL_{\paramsTwo}(\Pi_1)\not=\DL_{\paramsTwo}(\Pi_2)$ but for all
    $\params$, $\WL_{\params}(\Pi_1)=\WL_{\params}(\Pi_2)$.
\end{theorem}


\begin{proof}[Proof]
[$\exists$$<$]
\begin{figure}
\renewcommand{\xxshift}{0.7cm}
\renewcommand{\yyshift}{0.9cm}
\renewcommand{\csize}{0.8cm}
\renewcommand{\descccsize}{\scriptsize}
\newcommand{\objectsize}{\descccsize}
\renewcommand{\ilgword}[1]{\text{\descccsize{#1}}}
\newcommand{\objword}[1]{\text{\objectsize{#1}}}
\renewcommand{\desccccc}{5cm}
\renewcommand{\desccc}{6cm}
\renewcommand{\anchor}{west}
\newcommand{\minimumheight}{1.5em}
\newcommand{\subcapa}{\ilg{} of $\Pi_1$}
\newcommand{\subcapb}{\ilg{} of $\Pi_2$}
\newcommand{\wholecap}{\ilg{} representations of tasks in Thm.~\ref{thm:wl_dl} [$\exists$$<$].}
\centering
\begin{subfigure}{0.49\columnwidth}
    \centering
    \begin{tikzpicture}[  
        thick,  
        every node/.style={  
            draw,  
            rounded corners,  
            minimum height=\minimumheight,  
        }  
    ]
    \node[fill=caribbeangreen] (onaa) at (-2.25*\xxshift,\yyshift) {
    $\ilgword{Q(a,a)}$
    };
    \node[fill=caribbeangreen] (onbb) at (-0.75*\xxshift,\yyshift) {
    $\ilgword{Q(b,b)}$
    };
    \node[fill=yellow] (onab) at ( 0.75*\xxshift,\yyshift) {
    $\ilgword{W(a,b)}$
    };
    \node[fill=yellow] (onba) at ( 2.25*\xxshift,\yyshift) {
    $\ilgword{W(b,a)}$
    };
    \node[fill=babyblue] (a) at (  -1.5*\xxshift,2*\yyshift) {
    $\objword{a}$
    };
    \node[fill=babyblue] (b) at (   1.5 *\xxshift,2*\yyshift) {
    $\objword{b}$
    };
    
    \path [-,draw=plta] (a.south) edge[bend left=15] (onaa.north);
    \path [-,draw=pltb] (a.south) edge[bend right=15] (onaa.north);
    \path [-,draw=plta] (b.south) edge[bend left=15] (onbb.north);
    \path [-,draw=pltb] (b.south) edge[bend right=15] (onbb.north);
    
    \path [-,draw=plta] (a.south) edge (onab.north);
    \path [-,draw=pltb] (b.south) edge (onab.north);
    
    \path [-,draw=plta] (b.south) edge (onba.north);
    \path [-,draw=pltb] (a.south) edge (onba.north);

    \end{tikzpicture}
    \caption{\subcapa}
\end{subfigure}
\begin{subfigure}{0.49\columnwidth}
    \centering
    \begin{tikzpicture}[  
        thick,  
        every node/.style={  
            draw,  
            rounded corners,  
            minimum height=\minimumheight,  
        }  
    ]
    \node[fill=caribbeangreen] (onaa) at (-2.25*\xxshift,\yyshift) {
    $\ilgword{Q(a,b)}$
    };
    \node[fill=caribbeangreen] (onbb) at (-0.75*\xxshift,\yyshift) {
    $\ilgword{Q(b,a)}$
    };
    \node[fill=yellow] (onab) at ( 0.75*\xxshift,\yyshift) {
    $\ilgword{W(a,b)}$
    };
    \node[fill=yellow] (onba) at ( 2.25*\xxshift,\yyshift) {
    $\ilgword{W(b,a)}$
    };
    \node[fill=babyblue] (a) at (  -1.5*\xxshift,2*\yyshift) {
    $\objword{a}$
    };
    \node[fill=babyblue] (b) at (   1.5 *\xxshift,2*\yyshift) {
    $\objword{b}$
    };
    
    \path [-,draw=plta] (a.south) edge (onaa.north);
    \path [-,draw=pltb] (b.south) edge (onaa.north);
    \path [-,draw=plta] (b.south) edge (onbb.north);
    \path [-,draw=pltb] (a.south) edge (onbb.north);
    
    \path [-,draw=plta] (a.south) edge (onab.north);
    \path [-,draw=pltb] (b.south) edge (onab.north);
    
    \path [-,draw=plta] (b.south) edge (onba.north);
    \path [-,draw=pltb] (a.south) edge (onba.north);

    \end{tikzpicture}
    \caption{\subcapa}
\end{subfigure}
\caption{\wholecap}  
\label{fig:ilg_counterexample}
\end{figure}

    We begin by describing a pair of planning tasks $\Pi_1$ and $\Pi_2$ such
    that $\WL_{\params}(\Pi_1)=\WL_{\params}(\Pi_2)$ for any set of parameters
    $\params$ but are distinguished by \DescriptionLogic{}.    
    Let $\Pi_1 = \probliftedgeneral{s_0^1}{G}$ and $\Pi_2 =
    \probliftedgeneral{s_0^2}{G}$ with 
    $\predicates = \set{Q, W}$,
    $\objects=\set{a, b}$, 
    $\schemata$ contains the single action schema 
    $o = \langle\{x,y\}$, $\{Q(x, y)\}$, $\{W(x, y)\}$, $\emptyset\rangle$, 
    $G = \{W(a, b)$, $W(b, a) \}$,
    $s_0^1 = \{Q(a,a)$, $Q(b,b)\}$ and 
    $s_0^2 = \{Q(a, b)$, $Q(b, a)\}.$
    
    We have that $h^*(\Pi_1) = \infty$ as the problem $\Pi_1$ is unsolvable,
    while $h^*(\Pi_2) = 2$ as an optimal plan contains actions $o(a, b)$ and
    $o(b, a)$.
    DL features are able to distinguish the two planning tasks by considering
    the role-value map $(Q = W)(s)$ defined by $\set{x \mid \forall y : Q(x, y)
    \in s \iff W(x, y) \in s}$,
    and the corresponding numerical feature $\abs{Q=W}(s) = \abs{(Q=W)(s)}$.
    We have that $\abs{Q=W}(s_0^1) = 0$ and $\abs{Q=W}(s_0^1) = 2$,
    meaning that $\DL$ can distinguish between $\Pi_1$ and $\Pi_2$.
    
    On the other hand, the \ilg{} representations of $\Pi_1$ and $\Pi_2$ are
    indistinguishable to our definition of the edge-labelled WL algorithm.
    Fig.~\ref{fig:ilg_counterexample} illustrates this example and we note that
    it is similar to the implicit Muninn graph representations of the pair of
    planning tasks from Thm.~\ref{thm:gnn_muninn}.
    
    [$\exists$$>$]
    We identify a pair of problems with ternary predicates which compile to the
    same problem with only binary predicates for which DL features are defined.
    For problems with at most binary predicates, DL introduces base roles on
    each predicate $P(x, y) \in \predicates$ by $P^s = \set{(a, b) \mid P(a, b)
    \in s}$ where $s$ is a planning state.
    Then given an $n$-ary predicate $R(x_1, \ldots, x_n)$, DL introduces
    $n(n-1)/2$ roles defined by $R_{i,j}^s = \{(a, b) \mid$
    $\exists o_1,\ldots,o_{i-1},o_{i+1},\ldots,o_{j-1},o_{j+1},\ldots,o_n \in
    \objects$,
    $R(o_1,\ldots,o_{i-1},a,o_{i+1},\ldots,o_{j-1},b,o_{j+1},\ldots,o_n) \in
    s\}$
    for $1 \leq i < j \leq n$.
    Now consider the problems $\Pi_1 = \probliftedgeneral{s_0^1}{G}$ and
    $\Pi_2 = \probliftedgeneral{s_0^2}{G}$ now with 
    $\predicates = \{P\}$,
    $\objects = \{a, b, c, d\}$,
    $\schemata = \emptyset$,
    $G = \{P(a,b,c)\}$,
    and 
    \begin{align*}
        s_0^1 &= \{P(a,b,a), P(c,b,c), P(a,d,c), P(c,d,a)\} \\
        s_0^2 &= \{P(a,b,c), P(c,b,a), P(a,d,a), P(c,d,c)\}.
    \end{align*}
    
    We have that $h^*(\Pi_1) = \infty$ since there are no actions and the
    initial state is not the goal condition, while $h^*(\Pi_2) = 0$ since $G
    \subseteq s_0^2$.
    The \ilg{} for the two tasks are distinguished by the WL algorithm as the
    \ilg{} of $\Pi_1$ has no achieved goal colour while $\Pi_2$ does.
    However, DL features view the two states $s_0^1$ and $s_0^2$ as the
    same due after the compilation from ternary to binary predicates:
    \begin{align*}
        \begin{array}{c c c c}
            P_{1,2}(a,b) & P_{1,2}(a,d) & P_{1,2}(c,b) & P_{1,2}(c,d) \\
            P_{1,3}(a,a)&
            P_{1,3}(a,c)&
            P_{1,3}(c,a)&
            P_{1,3}(c,c)\\
            P_{2,3}(b,a)&
            P_{2,3}(b,c)&
            P_{2,3}(d,a)&
            P_{2,3}(d,c).
        \end{array}
    \end{align*}
    
    Thus any DL features will be the same for both $s_0^1$ and $s_0^2$
    and thus cannot distinguish $\Pi_1$ and $\Pi_2$.
\end{proof}

Our final theorem combines previous results and states that there exist domains for which all feature generators defined thus far are not powerful enough to perfect learn $h^*$, with proof in the appendix.
Although this is not a surprising result, we hope to bring intuition on
what is further required for constructing more expressive planning features.

\begin{corollary}
    [All feature generation models thus far cannot generate features that allow us
    to learn $h^*$ for all domains.]  
    \label{thm:wl_cannot}
    Let $\mathcal{F} \in \left\{ \WL, \GNN, \MUNINN, \DL \right\}$.
    There exists a domain $\Domain$ with a pair of planning tasks $\Pi_1$, $\Pi_2$
    such that for all parameters $\params$ for $\mathcal{F}$, we have that
    $\mathcal{F}_{\params}(\Pi_1) = \mathcal{F}_{\params}(\Pi_2)$ and $h^*(\Pi_1)
    \not= h^*(\Pi_2)$.
\end{corollary}

In this section, we concluded that our $\WL$ features are one of the most
expressive features thus far in the literature for representing planning tasks,
the other being $\DL$ features.
We have done so by drawing an expressivity hierarchy between our $\WL$ features
and previous work on GNN architectures~\cite{stahlberg:etal:2022}.
We further constructed explicit counterexamples illustrating the difference
between $\WL$ and $\DL$ features, highlighting their respective advantages and
limitations.

\section{Experiments}
In this section, we empirically evaluate WL-GOOSE\footnote{Source code available at ~\cite{chen:etal:2024zenodo}} for learning domain-specific heuristics using WL features against the state-of-the-art.
We consider the domains and training and test sets from the learning track of the 2023 International Planning Competition (IPC)~\cite{seipp:segovia:2023}.
The domains are Blocksworld, Childsnack, Ferry, Floortile, Miconic, Rovers, Satellite, Sokoban, Spanner, and Transport.
Actions in all domains have unit cost.
Each domain contains instances categorised into easy, medium and hard difficulties depending on the number of objects in the instance.
For each domain, the training set consists of at most 99 easy instances and the test set consists of exactly 30 instances from each of the three easy, medium and hard difficulties that are not in the training set.

The hyperparameters considered for WL-GOOSE are the number of iterations $L$ for generating features using the WL algorithm and the choice of a machine learning model used and its corresponding hyperparameters.
In all our experiments with WL-GOOSE, we use $L=4$ and, since our learning target is $h^*$, we consider the following regression models:
support vector regression with the dot product kernel (\svr{}) and the radial basis kernel (\svrInf{}), and Gaussian process regression~\cite{rasmussen:williams:2006} with the dot product kernel (\gpr).
We choose \svr{} over ridge regression for our kernelised linear model due to its sparsity and hence faster evaluation time with use of the $\epsilon$-insensitive loss function~\cite{vapnik:2000}.
The choice of Gaussian process allows us to explore a Bayesian treatment for learning $h^*$, providing us with confidence bounds on learned heuristics.

Furthermore, we experiment with the \textit{2-LWL algorithm}~\cite{morris:etal:2017} with $L=4$ for generating features alongside SVR with the dot product kernel (\lwlTwo).
The 2-LWL algorithm is a computationally feasible approximation of the 2-WL algorithm~\cite{morris:etal:2017}, which in turn is a generalisation of the WL algorithm where colours are assigned to pairs of vertices.
While the features computed by the 2-WL algorithm subsume those of the WL algorithm, it is requires quadratically more time than the WL algorithm.

For any configuration of WL-GOOSE, we use optimal plans returned by scorpion~\cite{seipp:etal:2020} on the training set with a 30-minute timeout on each instance for training.
States and the corresponding cost to the goal from each optimal plan are used as training data.
As baselines for heuristics, we use the domain-independent heuristic \hff{} and GNNs.
For the GNNs, we use GOOSE~\cite{chen:etal:2024} operating on the ILG representations of planning tasks with both max and mean aggregators and Muninn adapted to learn heuristics for use in GBFS only~\cite{stahlberg:2024}.
Every GOOSE model and SVR model is trained and evaluated 5 times with mean scores reported.
\gpr{}'s optimisation is deterministic and thus is only trained and evaluated once.
All GNNs use 4 message passing layers and a hidden dimension of 64.
All heuristics are evaluated using GBFS.
We include LAMA~\cite{richter:westphal:2010} using its first plan output as a strong satisficing planner baseline that uses multi-queue heuristic search and other optimisation techniques.
All methods use a timeout of 1800 seconds per evaluation problem.
Non-GNN models were run on a cluster with single Intel Xeon 3.2 GHz CPU cores and a memory limit of 8GB.
GOOSE used an NVIDIA RTX A6000 GPU, and Muninn an NVIDIA A10.
Other competition planners were not considered because they do not learn a heuristic.

Tab.~\ref{tab:summary} summarises our results with the coverage per domain for each planner and their total IPC score.
We discuss our results in detail and conclude this section by describing how to analyse the learned features of our models using an example.
More results can be found in the appendix.

\renewcommand{\arraystretch}{.9}
\begin{table}[t]
\setlength{\tabcolsep}{3pt}
\centering
\tablesize
\newcommand{\allfirst}[1]{{\underline{#1}}}

\begin{tabularx}{\columnwidth}{l Y Y Y Y Y Y Y Y Y} \toprule 
&
\multicolumn{2}{c}{classical} 
&
\multicolumn{3}{c}{GNN}
&
\multicolumn{4}{c}{WLF}
\\
\cmidrule(lr){2-3}  
\cmidrule(lr){4-6}  
\cmidrule(lr){7-10}  
Domain 
& \header{LAMA-F} 
& \header{\ff} 
& \header{\muninn} 
& \header{GOOSE$^{\ddagger}_{\text{max}}$} 
& \header{GOOSE$^{\ddagger}_{\text{mean}}$} 
& \header{\svrOne\seeded} 
& \header{\svr$^{\ddagger}_{\infty}$} 
& \header{\svr$^{\ddagger}_{\text{2-LWL}}$} 
& \header{\gprNew} 

\\ \midrule
blocksworld & 61 & \normalcell{28}{} & \normalcell{53}{} & \third{63.0}{} & \normalcell{60.6}{} & \second{72.2}{} & \normalcell{19.0}{} & \normalcell{22.2}{} & \first{\allfirst{75}}{}\\
childsnack & \allfirst{35} & \second{26}{} & \normalcell{12}{} & \normalcell{23.2}{} & \normalcell{15.6}{} & \third{25.0}{} & \normalcell{13.0}{} & \normalcell{9.8}{} & \first{29}{}\\
ferry & 68 & \normalcell{68}{} & \normalcell{38}{} & \third{70.0}{} & \third{70.0}{} & \first{\allfirst{76.0}}{} & \normalcell{32.0}{} & \normalcell{60.0}{} & \first{\allfirst{76}}{}\\
floortile & 11 & \first{\allfirst{12}}{} & \normalcell{1}{} & \normalcell{0.0}{} & \normalcell{1.0}{} & \second{2.0}{} & \normalcell{0.0}{} & \normalcell{0.0}{} & \second{2}{}\\
miconic & \allfirst{90} & \first{\allfirst{90}}{} & \first{\allfirst{90}}{} & \normalcell{88.6}{} & \normalcell{86.8}{} & \first{\allfirst{90.0}}{} & \normalcell{30.0}{} & \normalcell{67.0}{} & \first{\allfirst{90}}{}\\
rovers & \allfirst{67} & \third{34}{} & \normalcell{24}{} & \normalcell{25.6}{} & \normalcell{28.8}{} & \first{37.6}{} & \normalcell{28.0}{} & \normalcell{33.6}{} & \second{37}{}\\
satellite & \allfirst{89} & \first{65}{} & \normalcell{16}{} & \normalcell{31.0}{} & \normalcell{27.4}{} & \third{46.0}{} & \normalcell{29.4}{} & \normalcell{19.0}{} & \second{53}{}\\
sokoban & \allfirst{40} & \third{36}{} & \normalcell{31}{} & \normalcell{33.0}{} & \normalcell{33.4}{} & \first{38.0}{} & \normalcell{30.0}{} & \normalcell{30.6}{} & \first{38}{}\\
spanner & 30 & \normalcell{30}{} & \first{\allfirst{76}}{} & \normalcell{46.4}{} & \normalcell{36.6}{} & \second{73.2}{} & \normalcell{30.0}{} & \normalcell{51.8}{} & \third{73}{}\\
transport & \allfirst{66} & \first{41}{} & \normalcell{24}{} & \normalcell{32.4}{} & \second{38.0}{} & \normalcell{30.6}{} & \normalcell{27.0}{} & \third{34.2}{} & \normalcell{29}{}\\
\midrule
all & \allfirst{557} & \third{430}{} & \normalcell{365}{} & \normalcell{413.2}{} & \normalcell{398.2}{} & \second{490.6}{} & \normalcell{238.4}{} & \normalcell{328.2}{} & \first{502}{}\\
IPC score & \allfirst{492.7} & \third{393.5}{} & \normalcell{328.9}{} & \normalcell{391.0}{} & \normalcell{372.8}{} & \second{453.7}{} & \normalcell{210.7}{} & \normalcell{297.8}{} & \first{461.3}{}\\
\bottomrule

\end{tabularx}

\caption{Coverage of planners.
The bottom-most row provides their overall IPC 2023 learning track score.
Our new models are the WLF models.
Models marked $\ddagger$ are run 5 times with mean scores presented.
\lama{} is the only planner not performing single-queue GBFS.
The top three single-queue heuristic search planners in each row are indicated
by the cell colouring intensity, with the best one in bold.
The best planner overall in each row is underlined.
}\label{tab:summary}
\end{table}

\subsubsection*{How well do heuristics learned from WL features perform?}
Considering total coverage and total IPC score (Tab.~\ref{tab:summary}), we notice that \svr{} and \gpr{} outperform all the other planners with the exception of \lama{}, i.e., all learning-based approaches as well as \hff{}.
Domain-wise, both \svr{} and \gpr{} outperform Muninn and GOOSE on 9 domains.
Both \svr{} and \gpr{} outperform or tie with LAMA on 4 domains, namely Blocksworld, Ferry, Miconic and Spanner.
\gpr{} is able to return better plans than LAMA on 5 domains (Blocksworld, Childsnack, Ferry, Miconic, Sokoban), while the reverse is true only on 3 domains (Rovers, Satellite, Spanner).
For Spanner, this is because LAMA's heuristics are not informative for this domain, which leads to it performing like blind search and hence returning better plans on problems it can solve.

\newcommand{\smlversuswidth}{0.39\columnwidth}
\begin{figure}[t]
\centering
\raisebox{-0.5\height}{\includegraphics[width=\smlversuswidth]{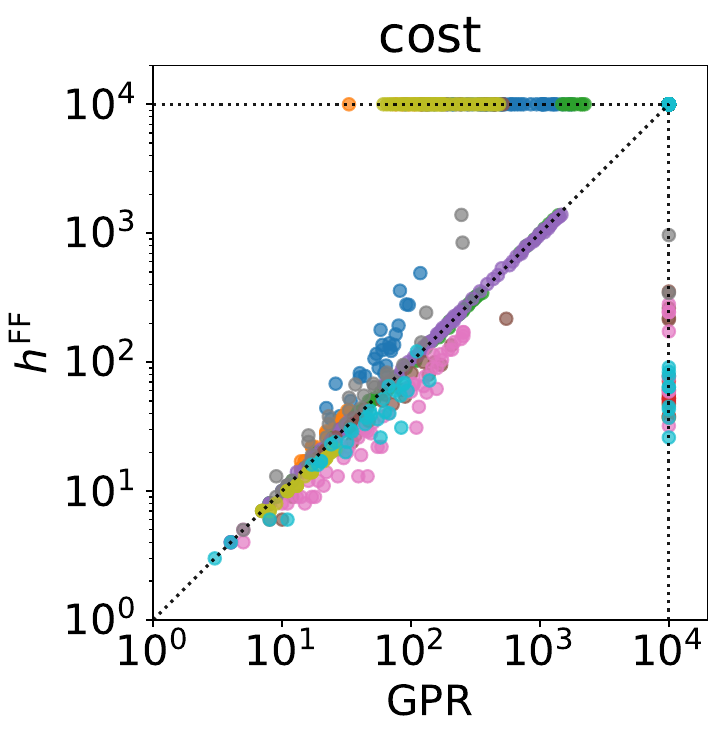}}
\raisebox{-0.5\height}{\includegraphics[width=\smlversuswidth]{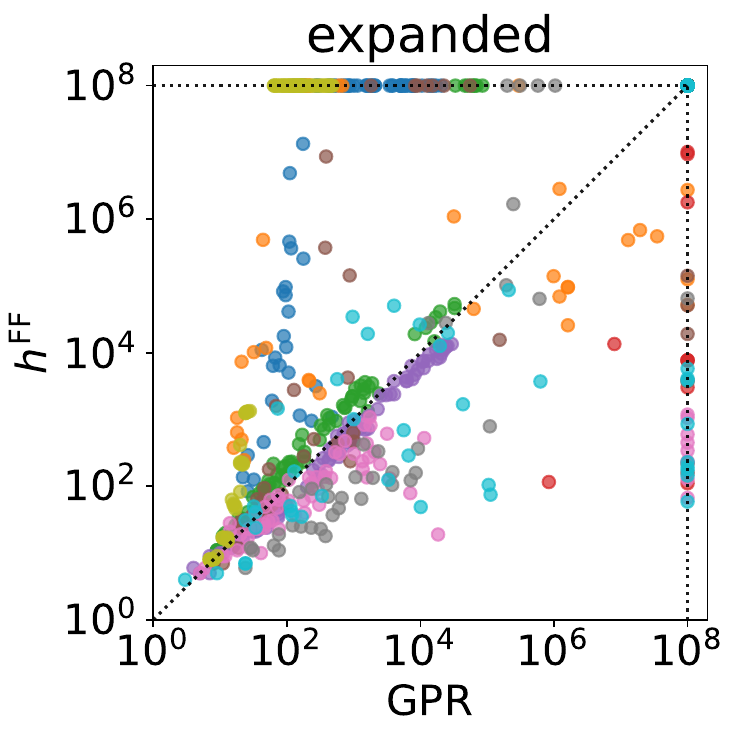}}
\raisebox{-0.45\height}{\includegraphics[width=0.2\columnwidth]{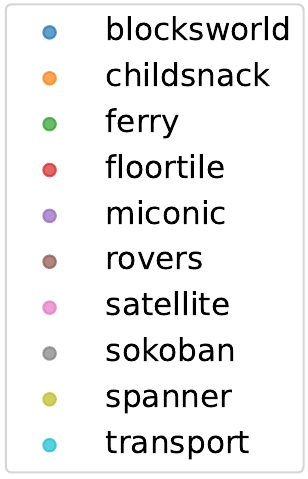}}
\caption{Returned plan cost and number of expanded nodes of \hff{} and \gpr{}.
Problems that were not solved by one planner has their respective metric set to
the axis limit.
Points on the top left triangle favour \gpr{} while points on the bottom right
triangle favour \hff{}.
}
\label{fig:gpr_vs_hff}
\end{figure}

\svr{} and \gpr{} also outperform or tie with \hff{} on 6 and 7 domains, respectively.
We compare \gpr{} and \hff{} in more detail in Fig.~\ref{fig:gpr_vs_hff} by showing plan cost and nodes expanded per problem.
We observe that the better performing planner on a domain generally has better plan quality and fewer node expansions.
An exception is Sokoban where \gpr{} expands more nodes but solves more problems due to its faster heuristic evaluations.
Overall, the domains in which \gpr{} performs worse are domains that require traversing a map which WL features cannot express with limited iterations.

\subsubsection*{Are our methods more computationally efficient to train?}
To answer this question, we compare the training time of GNNs using ILGs, \svr{} and \gpr{}.
Their mean and standard deviation in seconds are 
$77.2\pm33.7$ (\gooseMax), 
$33.7\pm52.6$ (\gooseMean), 
$0.2\pm0.1$ (\svr) and 
$3.8\pm4.6$ (\gpr).
Comparing against the more efficient GNN model per domain, we have that \svr{} is between 187x (Satellite) to 922x (Childsnack) more efficient and \gpr{} is between 8x (Floortile) and 615x (Childsnack) more efficient.
Note that the GNNs have access to GPUs and would take even more time to train on a CPU.
Lastly, the GNNs have between $54529$ and $74561$ number of parameters in the range, while the WLF models have between $108$ and $23202$ parameters.

\subsubsection*{Does kernelising help?}
As commonly done in classical machine learning, we combine our WL features with non-linear kernels to obtain new non-linear features that can increase the expressivity of the regression models.
Unfortunately, as shown in Tab.~\ref{tab:summary}, this generally results in a decrease in the performance of the learned heuristic:
the \svrInf{} model has significantly worse coverage than \svr{} despite theoretically having more expressive implicit features.
The drop in performance can be explained by overfitting to the more expressive features which do not bring any obvious semantic information for planning tasks.

\subsubsection*{Do higher order WL features help?}
The motivation for using higher-order WL features is similar to using higher-order kernels: to introduce more expressive features that may be correlated with the optimal heuristic.
In Tab.~\ref{tab:summary}, we see that the performance of 2-LWL is generally worse on all domains except for Transport.
This again can be attributed to poorer generalisation.
Furthermore, computing the 2-LWL features are slower to generate than WL features as they take time cubic in the size of the ILGs in the worst case, and in the case of Floortile runs out of memory when generating features.
We also note that attempting to generate 3-LWL features causes out of memory problems during training as the size of features generated is extremely large, on the order of $10^7$ and above.

\begin{table}[t]
\tablesize
\setlength{\tabcolsep}{2pt}
\begin{tabularx}{\columnwidth}{l Y Y Y Y Y Y Y Y }
\toprule
& \multicolumn{4}{c}{$h$ error}  
& \multicolumn{4}{c}{Expanded}  
\\
\cmidrule(lr){2-5}  
\cmidrule(lr){6-9}  
Domain  
& easy & medium & hard & all  
& easy & medium & hard & all  
\\
\midrule
blocksworld & \significant{\highcorr{+0.93}} & \significant{\highcorr{+0.90}} & \significant{\highcorr{+0.94}} & \significant{\highcorr{+0.98}} & \notsignificant{\medcorr{+0.32}} & \notsignificant{\lowcorr{+0.22}} & \notsignificant{\medcorr{+0.33}} & \significant{\highcorr{+0.58}}\\
childsnack & \significant{\highcorr{+0.69}} & \significant{\highcorr{+0.93}} & \multicolumn{1}{c}{-} & \significant{\highcorr{+0.87}} & \significant{\highcorr{+0.59}} & \notsignificant{\highcorr{+0.52}} & \multicolumn{1}{c}{-} & \notsignificant{\lowcorr{+0.20}}\\
ferry & \significant{\highcorr{+0.86}} & \significant{\highcorr{+0.98}} & \significant{\highcorr{+0.99}} & \significant{\highcorr{+1.00}} & \significant{\highcorr{+0.86}} & \significant{\highcorr{+0.87}} & \significant{\highcorr{+0.83}} & \significant{\highcorr{+0.93}}\\
floortile & \multicolumn{1}{c}{-} & \multicolumn{1}{c}{-} & \multicolumn{1}{c}{-} & \multicolumn{1}{c}{-} & \multicolumn{1}{c}{-} & \multicolumn{1}{c}{-} & \multicolumn{1}{c}{-} & \multicolumn{1}{c}{-}\\
miconic & \significant{\highcorr{+0.56}} & \significant{\highcorr{+0.67}} & \significant{\highcorr{+0.97}} & \significant{\highcorr{+0.96}} & \significant{\highcorr{+0.55}} & \significant{\highcorr{+0.81}} & \significant{\highcorr{+0.99}} & \significant{\highcorr{+0.99}}\\
rovers & \significant{\highcorr{+0.89}} & \significant{\highcorr{+0.86}} & \multicolumn{1}{c}{-} & \significant{\highcorr{+0.96}} & \notsignificant{\lowcorr{+0.26}} & \notsignificant{\lowcorr{+0.19}} & \multicolumn{1}{c}{-} & \significant{\highcorr{+0.53}}\\
satellite & \significant{\highcorr{+0.73}} & \significant{\highcorr{+0.95}} & \multicolumn{1}{c}{-} & \significant{\highcorr{+0.96}} & \notsignificant{\lowcorr{+0.09}} & \notsignificant{\lowcorr{+0.07}} & \multicolumn{1}{c}{-} & \notsignificant{\lowcorr{+0.18}}\\
sokoban & \notsignificant{\lowcorr{+0.27}} & \significant{\highcorr{+0.86}} & \multicolumn{1}{c}{-} & \significant{\highcorr{+0.96}} & \notsignificant{\lowcorr{+0.26}} & \significant{\highcorr{+0.76}} & \multicolumn{1}{c}{-} & \significant{\highcorr{+0.79}}\\
spanner & \significant{\medcorr{+0.36}} & \significant{\highcorr{+0.53}} & \significant{\highcorr{+0.96}} & \significant{\highcorr{+0.92}} & \significant{\medcorr{+0.43}} & \significant{\highcorr{+0.54}} & \significant{\highcorr{+0.96}} & \significant{\highcorr{+0.92}}\\
transport & \significant{\highcorr{+0.83}} & \multicolumn{1}{c}{-} & \multicolumn{1}{c}{-} & \significant{\highcorr{+0.83}} & \significant{\medcorr{+0.37}} & \multicolumn{1}{c}{-} & \multicolumn{1}{c}{-} & \significant{\medcorr{+0.35}}\\
\bottomrule

\end{tabularx}
\caption{
Pearson's correlation coefficient $\rho$ between standard deviation obtained by
\gpr{} against 
heuristic estimate error and node expansions of initial states from solved problems.
Statistically significant coefficients ($p<0.05$) are highlighted in bold font
and italics otherwise.
Strongly correlated values ($\abs{\rho} \geq 0.5$) are highlighted in green,
medium correlated values ($0.3 \leq \abs{\rho} < 0.5$) in a lighter green, and
low correlation ($\abs{\rho} < 0.3$) in gray.
Entries for which we solved fewer then 10 problems are omitted.
}
\label{tab:correlation}
\end{table}

\subsubsection*{Are Bayesian variance estimates meaningful?}
One advantage of Bayesian models is that by assuming a prior distribution on the weights of our models, we are able to derive uncertainty bounds on the outputs of the learned posterior model.
In Tab.~\ref{tab:correlation}, we analyse the Pearson's correlation coefficient between the standard deviation obtained by \gpr{} and
(1) the error between output mean and $h^*$, and
(2) the number of expanded nodes using the learned heuristic with greedy best first search.
We see that there is a statistically significant strong correlation between the heuristic estimate error and the \gpr{} variance outputs.
This is reasonable given that the derivation of the Bayesian model computes the uncertainty on its output prediction.
The story is different for the number of expansions during search where for easy problems there is no significant correlation depending on the domain.
Interestingly, the correlation is more significant and stronger on harder problems for more domains.
Thus, the Bayesian model is able to determine the difficulty of solving a problem within a domain by looking at the predicted standard deviation for $h(s_0)$ but the quality of this prediction will depend on the domain.

\newcommand{\colA}{c_0}
\newcommand{\colB}{c_1}
\newcommand{\colC}{c_2}
\newcommand{\colD}{c_3}
\newcommand{\colE}{c_4}
\newcommand{\colF}{c_5}
\newcommand{\colG}{c_6}
\newcommand{\colH}{c_7}
\newcommand{\colI}{c_8}

\begin{figure}
    \centering
    \newcommand{\linebreakspacing}{-0.25em}
    \newcommand{\linebreakspacingb}{-0.5em}
    \renewcommand{\xxshift}{1.4}
    \renewcommand{\yyshift}{-1.4}
    \scriptsize
    \begin{tikzpicture}[  
        thick,  
        every node/.style={  
            draw,  
            minimum height=0.9cm,  
            minimum width=2.5cm,   
            align=left, }  
    ]

    \node (a) at (-2*\xxshift,\yyshift) { $\colA{}:
        (\truegoalsymbol,\! \text{on-table})$\\[\linebreakspacing]
        achieved on-table goal 
    }; 
        
    \node (b) at ( 0*\xxshift,\yyshift) { $\colB{}:
        \objectsymbol$ \\[\linebreakspacing]
        block 
    }; 
        
    \node (c) at (+2*\xxshift,\yyshift) { $\colC{}:
        (\truegoalsymbol,\! \text{on})$ \\[\linebreakspacing]
        achieved on goal 
    };

    \node (d) at (-2*\xxshift,2*\yyshift) { $\colD{}:
        (\colA{},\!\msetaa{(\colB{},\!0)}\!)$\\[\linebreakspacing]
        achieved on-table goal 
    }; 
        
    \node (e) at (-0.05*\xxshift,2*\yyshift) { $\colE{}\!\!:\!\!
        (\colB{},\!\msetaa{(\colA{},\!0),\!(\colC{},\!1)}\!)$ \\[\linebreakspacing]
        block correctly on table\\[\linebreakspacingb]
        with correct block above 
    }; 
        
    \node (f) at (+2*\xxshift,2*\yyshift) { $\colF{}\!\!:\!\!
        (\colC{},\!\msetaa{(\colB{},\!0),\!(\colB{},\!1)}\!)$ \\[\linebreakspacing]
        achieved on goal 
    };

    \node (g) at (-1.2*\xxshift,3*\yyshift) { $\colG{}:
        (\colD{},\!\msetaa{(\colE{},\!0)}\!)$ \\[\linebreakspacing]
        block correctly on table\\[\linebreakspacingb]
        with correct block above 
    }; 
        
    \node (h) at ( 1.2*\xxshift,3*\yyshift) { $\colH{}:
        (\colE{},\!\msetaa{(\colD{},\!0),\!(\colF{},\!1)}\!)$ \\[\linebreakspacing]
        block correctly on table\\[\linebreakspacingb]
        with correct block above 
    };

    \node (i) at (0,4*\yyshift) { $\colI{}:
        (\colG{},\!\msetaa{(\colH{},\!0)}\!)$ \\[\linebreakspacing]
        block correctly on table\\[\linebreakspacingb]
        with correct block above 
    };

    \path[-Stealth,draw] (a) -- (d); 
    \path[-Stealth,draw] (a) -- (e); 
    \path[-Stealth,draw] (b) -- (d); 
    \path[-Stealth,draw] (b) -- (e); 
    \path[-Stealth,draw] (b) -- (f); 
    \path[-Stealth,draw] (c) -- (e); 
    \path[-Stealth,draw] (c) -- (f); 
    \path[-Stealth,draw] (d) -- (g);
    \path[-Stealth,draw] (d) -- (h); 
    \path[-Stealth,draw] (e) -- (g); 
    \path[-Stealth,draw] (e) -- (h); 
    \path[-Stealth,draw] (f) -- (h); 
    \path[-Stealth,draw] (g) -- (i); 
    \path[-Stealth,draw] (h) -- (i);

    \end{tikzpicture}
\caption{
    The dependency subgraph of generated WL features on Blocksworld.
    The first row of each node indicates the feature colour, followed by the initial colour the feature corresponds to or the input to the $\hash$ function which generated the colour.
    The second row describes the semantic meaning of the feature.
    Edges describe the dependency of the feature on previous features based
    on the $\hash$ function.
}
\label{fig:achieved_goal_feature}
\end{figure}
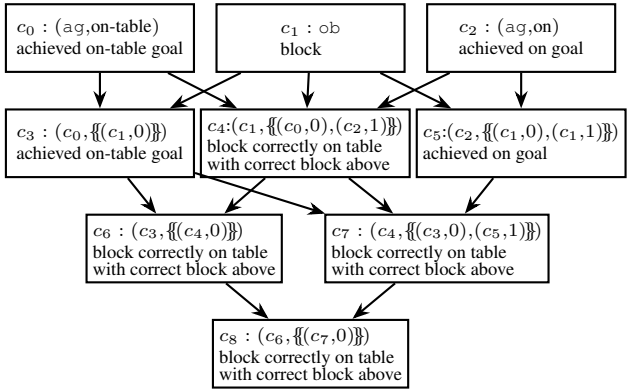

\subsection*{Understanding Learned Models}
Another advantage of WL-GOOSE is that its set of features is explainable, and
it is possible to see which features are chosen when using a linear inference
model.
The models can be understood by analysing the features with the highest
corresponding linear weights, and by observing the distribution of such
weights.
The semantic meaning of the features can be understood by examining the \textit{generation} of WL colours.
This can be achieved by representing the observed WL colours as a directed acyclic graph (DAG) where each WL colour is a node and there is a directed edge from $\kappa$ to $\kappa'$ if $\kappa' = \hash(x, M)$ and $x = \kappa$ or $\exists \iota, (\kappa, \iota) \in M$.
We provide an example of how to interpret the learned models by briefly studying the learned \gpr{} model on Blocksworld.
In this domain, a total of 10444 features were generated from the training data and Fig.~\ref{fig:achieved_goal_feature} illustrates the DAG representation of feature $\colI{}$'s generation.
Consider feature $\colE{}$ in Fig.~\ref{fig:achieved_goal_feature}, it computes the number of blocks that are correctly on the table and also have the correct block above it.
We have that $\colE{} = \hash(\colB{}, \mseta{(\colA{}, 0), (\colC{}, 1)})$, meaning that the colour $\colE{}$ is generated from an object node ($\colB{}=\objectsymbol$) which is part of an achieved on-table goal ($\colA{}=(\truegoalsymbol, \text{on-table})$) and achieved on goal ($\colC{}=(\truegoalsymbol, \text{on})$).
The corresponding edge label of the node colours indicate the position of the block object in the proposition indexed from 0.
Thus, blocks $b$ with colour $\colE{}$ are in the first and only argument of on-table and the second argument of on.
This means that the colour $\colE{}$ is assigned to blocks correctly on the table and correctly underneath another block.

Moreover, we observed that certain subsets of features were evaluated to the same value on all training states.
As a result, the same learned weight value was assigned to each feature in these subsets.
This can be seen in Fig.~\ref{fig:achieved_goal_feature} where features $\colE{}$, $\colH{}$, $\colG{}$ and $\colI{}$ are semantically equivalent. The sum of their weight values is $-1.76$, the largest in value from subsets of features.
Thus, the learned weight rewards states satisfying this condition as blocks correctly on the table do not have to be moved.

Note that it is possible for features to evaluate to the same values on the training set but have different semantic meanings because the training set is finite.
For example, in Blocksworld, a training set may satisfy that a block is correctly on the table if and only if it has the correct block above it.
In this case, the count of colours $\colA{}$ and $\colE{}$ would be the same on all states despite not being semantically equivalent.

\section{Conclusion}

We introduced WL-GOOSE, a novel approach that makes use of the efficiency of classical machine learning for learning to plan.
We developed the Instance Learning Graph (\ilg{}), a novel representation of lifted planning tasks and provided a method to generate features for ILGs based on the WL algorithm, agnostic to the downstream model. 
Similar to Description Logic Features for planning, our generated features are agnostic to the learning target and can be used without the need for backpropagation.
Furthermore, some of our models can be trained in a deterministic fashion with minimal parameter tuning in contrast to DL-based approaches.
To validate the benefits of WL-GOOSE, we used two classical SML models, support vector regression (\svr{}) and Gaussian process regression (\gpr{}), to learn domain-specific heuristics and compared them to the state of the art.

The experimental results showed that WL-GOOSE can efficiently and reliably learn domain-specific heuristics from scratch.
Compared to GNNs applied to ILGs, our learned heuristics are up to 3 orders of magnitude times faster to train and have up to 2 orders of magnitude fewer parameters.
Our results also showed that both \svr{} and \gpr{} are the first learned heuristics capable of outperforming \hff{} in terms of total coverage.
Moreover, our learned heuristics outperform or tie with LAMA on 4 domains. To our knowledge, this is the best performance of learned heuristics against LAMA.
We also showed the theoretical connections between our novel feature generation method with Description Logic Features and GNNs.
Our future work agenda includes exploring how to best use the uncertainty bounds provided by \gpr{} to improve search, making use of generated WL features for learning different forms of domain knowledge such as policies, landmarks and sketches~\cite{bonet:geffner:2021}, and combining stronger satisficing search algorithms to further improve the performance of WL-GOOSE.

\section*{Acknowledgements}
Many thanks must go to Simon St\r{a}hlberg for training and evaluating Muninn on GPUs with GBFS.
This work was supported by Australian Research Council grant DP220103815, by the Artificial and Natural Intelligence Toulouse Institute (ANITI) under the grant agreement ANR-19-PI3A-000, and by the European Union's Horizon Europe Research and Innovation program under the grant agreement TUPLES No. 101070149.

\small
\bibliography{references}

\clearpage
\normalsize
\appendix
\section{Summary of Theoretical Results}
This section constitutes for the proof of Cor.~\ref{thm:wl_cannot} by compiling the pairs of counterexamples that are distinguishable by the described feature generation methods for planning. To recall, we have for a set of given parameters $\params$:

\begin{itemize}
    \item $\WL_{\params}\!:\!\Domain \to \R^d$. The WL feature generation
    function described in Sec.~\ref{sec:method}
    \item $\GNN_{\params}\!:\!\Domain \to \R^d$. Message passing graph neural networks acting on ILG representations of planning tasks.
    \item $\DL_{\params}\!:\!\Domain \to \R^d$. Description Logic Feature generators~\cite{martin:geffner:2000}.
    \item $\MUNINN_{\params}\!:\!\Domain \to \R^d$. Architecture from~\citet{stahlberg:etal:2022} for generating features with random node initialisation removed.
\end{itemize}

\subsection{A.1. $\WL$}
From Thm.~\ref{thm:wl_dl}:
\begin{itemize}
    \item $\Pi_1 = \probliftedgeneral{s_0^1}{G}$
    \item $\Pi_2 = \probliftedgeneral{s_0^2}{G}$
    \item $\predicates = \set{Q, W}$
    \item $\objects=\set{a, b}$
    \item $\schemata=\set{o}$, $o = \langle\{x,y\}$, $\{Q(x, y)\}$, $\{W(x, y)\}$, $\emptyset\rangle$
    \item $G = \{W(a, b)$, $W(b, a) \}$
    \item $s_0^1 = \{Q(a,a)$, $Q(b,b)\}$
    \item $s_0^2 = \{Q(a, b)$, $Q(b, a)\}$
\end{itemize}
We have that $h^*(\Pi_1) = \infty$ and $h^*(\Pi_2) = 2$.

\subsection{A.2. $\GNN$}
Same counterexample as $\WL$ by Thm.~\ref{thm:wl_gnn}.

\subsection{A.3. $\DL$}
From Thm.~\ref{thm:wl_dl}:
\begin{itemize}
    \item $\Pi_1 = \probliftedgeneral{s_0^1}{G}$
    \item $\Pi_2 = \probliftedgeneral{s_0^2}{G}$
    \item $\predicates = \{P\}$
    \item $\objects = \{a, b, c, d\}$
    \item $\schemata = \emptyset$
    \item $G = \{P(a,b,c)\}$
    \item $s_0^1 = \{P(a,b,a), P(c,b,c), P(a,d,c), P(c,d,a)\}$
    \item $s_0^2 = \{P(a,b,c), P(c,b,a), P(a,d,a), P(c,d,c)\}$
\end{itemize}
We have that $h^*(\Pi_1) = \infty$ and $h^*(\Pi_2) = 0$.

\subsection{A.4. $\MUNINN = \GNNa^{\mug}$}
From Thm.~\ref{thm:gnn_muninn}:
\begin{itemize}
    \item $\Pi_1 = \probliftedgeneral{s_0^{(1)}}{G}$
    \item $\Pi_2 = \probliftedgeneral{s_0^{(2)}}{G}$
    \item $\predicates = \set{Q}$
    \item $\objects=\set{a, b}$
    \item $\schemata = \emptyset$
    \item $G = \set{Q(a, b), Q(b, a)}$
    \item $s_0^{(1)} = \set{Q(a, a), Q(b,b)}$
    \item $s_0^{(2)} = \set{Q(a, b), Q(b, a)}$
\end{itemize}
We have that $h^*(\Pi_1) = \infty$ and $h^*(\Pi_2) = 0$.

\section{More Detailed Evaluation Results}
\begin{table}[ht!]
    \setlength{\tabcolsep}{3pt}
    \tablesize
\newcommand{\allfirst}[1]{{\underline{#1}}}

\begin{tabularx}{\columnwidth}{l Y Y Y Y Y Y Y Y Y} \toprule 
    &
    \multicolumn{2}{c}{classical} 
    &
    \multicolumn{3}{c}{GNN}
    &
    \multicolumn{4}{c}{WLF}
    \\
    \cmidrule(lr){2-3}  
    \cmidrule(lr){4-6}  
    \cmidrule(lr){7-10}  
    Difficulty 
& \header{LAMA-F} 
& \header{\ff} 
& \header{\muninn} 
& \header{GOOSE$^{\ddagger}_{\text{max}}$} 
& \header{GOOSE$^{\ddagger}_{\text{mean}}$} 
& \header{\svrOne\seeded} 
& \header{\svr$^{\ddagger}_{\infty}$} 
& \header{\svr$^{\ddagger}_{\text{2-LWL}}$} 
& \header{\gprNew} 

    \\ \midrule
easy & 280 & 275 & 223 & 256.8 & 250.6 & 264.8 & 234.4 & 229.4 & 265\\
 &  &  &  & ±3.6 & ±2.6 & ±2.9 & ±2.1 & ±1.4 & \\
medium & 190 & 112 & 96 & 109.6 & 105.8 & 152.6 & 4.0 & 91.8 & 161\\
 &  &  &  & ±7.4 & ±10.3 & ±1.0 & ±0.0 & ±2.2 & \\
hard & 87 & 43 & 46 & 46.8 & 41.8 & 73.2 & 0.0 & 7.0 & 76\\
 &  &  &  & ±4.5 & ±2.6 & ±0.7 & ±0.0 & ±0.0 & \\
\midrule
all & 557 & 430 & 365 & 413.2 & 398.2 & 490.6 & 238.4 & 328.2 & 502\\
 &  &  &  & ±10.9 & ±9.8 & ±3.4 & ±2.1 & ±3.1 & \\
\bottomrule

\end{tabularx}

    \caption{
        Coverage of considered planners per difficulty level. 
        The mean and standard deviation are taken for models with multiple repeats marked by $\ddagger$.
    }
    \label{tab:difficulty}
\end{table}

\begin{table}[ht!]
    \setlength{\tabcolsep}{3pt}
    \tablesize
\newcommand{\allfirst}[1]{{\underline{#1}}}

\begin{tabularx}{\columnwidth}{l Y Y Y Y Y Y Y Y Y} \toprule 
&
\multicolumn{2}{c}{classical} 
&
\multicolumn{3}{c}{GNN}
&
\multicolumn{4}{c}{WLF}
\\
\cmidrule(lr){2-3}  
\cmidrule(lr){4-6}  
\cmidrule(lr){7-10}  
Domain 
& \header{LAMA-F} 
& \header{\ff} 
& \header{\muninn} 
& \header{GOOSE$^{\ddagger}_{\text{max}}$} 
& \header{GOOSE$^{\ddagger}_{\text{mean}}$} 
& \header{\svrOne\seeded} 
& \header{\svr$^{\ddagger}_{\infty}$} 
& \header{\svr$^{\ddagger}_{\text{2-LWL}}$} 
& \header{\gprNew} 

\\ \midrule
blocksworld & 39.8 & \normalcell{14.1}{} & \normalcell{46.1}{} & \third{60.5}{} & \normalcell{56.7}{} & \second{66.6}{} & \normalcell{17.7}{} & \normalcell{13.8}{} & \first{\allfirst{68.8}}{}\\
childsnack & 22.0 & \normalcell{20.1}{} & \normalcell{12.0}{} & \third{21.8}{} & \normalcell{14.8}{} & \second{23.3}{} & \normalcell{12.5}{} & \normalcell{7.9}{} & \first{\allfirst{26.9}}{}\\
ferry & 62.9 & \normalcell{67.6}{} & \normalcell{35.2}{} & \normalcell{69.8}{} & \third{69.9}{} & \first{\allfirst{75.8}}{} & \normalcell{31.0}{} & \normalcell{59.5}{} & \second{75.7}{}\\
floortile & 10.0 & \first{\allfirst{11.2}}{} & \normalcell{0.9}{} & \normalcell{0.0}{} & \normalcell{0.9}{} & \second{1.8}{} & \normalcell{0.0}{} & \normalcell{0.0}{} & \second{1.8}{}\\
miconic & 81.6 & \normalcell{88.5}{} & \first{\allfirst{90.0}}{} & \normalcell{87.7}{} & \normalcell{85.9}{} & \second{89.0}{} & \normalcell{29.3}{} & \normalcell{65.8}{} & \third{88.9}{}\\
rovers & \allfirst{65.2} & \first{32.7}{} & \normalcell{15.0}{} & \normalcell{21.7}{} & \normalcell{24.4}{} & \second{31.0}{} & \normalcell{23.6}{} & \normalcell{27.0}{} & \third{30.0}{}\\
satellite & \allfirst{87.3} & \first{63.8}{} & \normalcell{13.3}{} & \normalcell{23.9}{} & \normalcell{18.6}{} & \third{36.2}{} & \normalcell{20.5}{} & \normalcell{15.2}{} & \second{43.5}{}\\
sokoban & 29.8 & \normalcell{26.3}{} & \normalcell{21.3}{} & \normalcell{29.8}{} & \third{30.2}{} & \first{\allfirst{34.8}}{} & \normalcell{26.9}{} & \normalcell{27.9}{} & \second{34.7}{}\\
spanner & 30.0 & \normalcell{30.0}{} & \first{\allfirst{76.0}}{} & \normalcell{45.7}{} & \normalcell{36.2}{} & \second{69.2}{} & \normalcell{26.6}{} & \normalcell{47.7}{} & \third{69.0}{}\\
transport & \allfirst{64.0} & \first{39.3}{} & \normalcell{19.2}{} & \normalcell{30.1}{} & \second{35.2}{} & \normalcell{25.9}{} & \normalcell{22.8}{} & \third{32.9}{} & \normalcell{21.9}{}\\
\midrule
IPC score & \allfirst{492.7} & \third{393.5}{} & \normalcell{328.9}{} & \normalcell{391.0}{} & \normalcell{372.8}{} & \second{453.7}{} & \normalcell{210.7}{} & \normalcell{297.8}{} & \first{461.3}{}\\
\bottomrule

\end{tabularx}

    \caption{IPC quality scores of considered planners per domain.
    Models marked $\ddagger$ are run 5 times.
    The top three single-queue heuristic search planners in each row are indicated by the cell colouring intensity, with the best one in bold.
    The best planner overall in each row is underlined.}
    \label{tab:quality}
\end{table}

\begin{figure}[ht!]
    \centering
    \includegraphics[width=0.7\columnwidth]{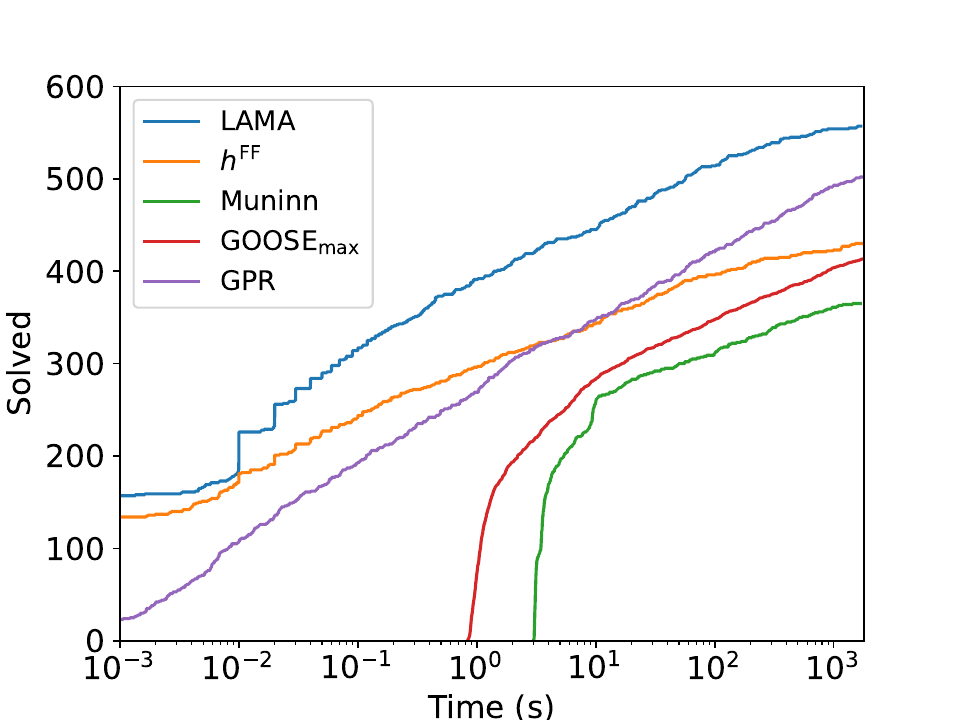}
    \caption{Cumulative coverage over time of selected solvers.}
\end{figure}

\newpage   
\section{Training Results}
\begin{table}[ht!]
    \setlength{\tabcolsep}{3pt}
    \tablesize
    \begin{tabularx}{\columnwidth}{l Y Y Y Y Y Y} \toprule 
&
\multicolumn{2}{c}{GNN}
&
\multicolumn{4}{c}{WLF}
\\
\cmidrule(lr){2-3}
\cmidrule(lr){4-7}
Domain
& \header{GOOSE$_{\text{max}}$} 
& \header{GOOSE$_{\text{mean}}$} 
& \header{\svrOne} 
& \header{\svr$_{\infty}$} 
& \header{\svr$_{\text{2-LWL}}$} 
& \header{\gprNew} 
\\ 
\midrule
blocksworld & 122.6 & 155.9 & 0.3 & 7.9 & 7.8 & 4.3 \\
childsnack & 36.9 & 46.0 & 0.0 & 3.1 & 0.2 & 0.1 \\
ferry & 56.1 & 112.2 & 0.2 & 6.5 & 3.4 & 1.5 \\
floortile & 122.2 & 146.0 & 0.5 & 28.3 & OOM & 14.9 \\
miconic & 46.5 & 50.2 & 0.1 & 95.3 & 0.6 & 0.5 \\
rovers & 100.4 & 88.4 & 0.3 & 18.5 & 7.2 & 7.5 \\
satellite & 49.7 & 97.6 & 0.3 & 13.1 & 7.7 & 1.8 \\
sokoban & 102.1 & 197.2 & 0.2 & 3.6 & 1.4 & 1.7 \\
spanner & 90.4 & 66.2 & 0.1 & 35.6 & 0.8 & 5.3 \\
transport & 44.7 & 41.9 & 0.1 & 0.6 & 2.1 & 0.5 \\
\midrule
all & 77.2 & 100.2 & 0.2 & 21.3 & 3.1 & 3.8 \\
    & ±33.7 & ±52.6 & ±0.1 & ±28.4 & ±3.2 & ±4.6 \\
\bottomrule
\end{tabularx}

    \caption{
        Average training time in seconds for various learning for planning models over each domain.
        The final row takes the standard deviation of training time over all domains.
        OOM denotes the memory limit was exceeded when trying to train the model.
    }
    \label{tab:train_time}
\end{table}

\begin{table}[ht!]
    \setlength{\tabcolsep}{3pt}
    \tablesize
    \begin{tabularx}{\columnwidth}{l Y Y Y} \toprule 
Domain & GNN & WL & improvement
\\ \midrule
blocksworld & 54721 & 10444 & 5 \\
childsnack & 56257 & 251 & 224 \\
ferry & 54529 & 3228 & 17 \\
floortile & 55681 & 7616 & 7 \\
miconic & 54913 & 108 & 508 \\
rovers & 74561 & 23202 & 3 \\
satellite & 55297 & 22155 & 2 \\
sokoban & 70913 & 110 & 645 \\
spanner & 54913 & 350 & 157 \\
transport & 54721 & 3787 & 14 \\
\bottomrule
\end{tabularx}

    \caption{
        Number of parameters per model class per domain, and the improvement on parameter efficiency of the WL models over the GNN models.
    }
    \label{tab:params}
\end{table}

\end{document}